\documentclass[journal,12pt,onecolumn,draftclsnofoot]{IEEEtran}
\usepackage{cite}
\usepackage{graphicx}
\usepackage{multirow}
\usepackage{rotating}
\usepackage{makecell}
\usepackage{algorithm}
\usepackage{algpseudocode}
\usepackage{array}
\usepackage[cmex10]{amsmath}
\usepackage{amssymb}
\usepackage{lipsum}
\usepackage{subfigure}
\usepackage{mathtools}
\usepackage{amsthm}
\newtheorem{mydef}{Definition}
\newtheorem{mytherm}{Theorem}
\newtheorem{remark}{Remark}
\usepackage[usenames, dvipsnames]{color}
\usepackage{booktabs}
\usepackage[switch,pagewise]{lineno}

\begin{document}
\title{Multiobjective Programming for Type-2 Hierarchical Fuzzy Inference Trees}

\author{Varun~Kumar~Ojha,~\IEEEmembership{Member,~IEEE,}
	V\'{a}clav~Sn\'{a}\v{s}el,~\IEEEmembership{Senior Member,~IEEE,}
	and~Ajith~Abraham,~\IEEEmembership{Senior Member,~IEEE}
	\thanks{V K Ojha is with the {\color{black}Chair of Information Architecture, ETH Zurich, Zurich, Switzerland e-mail: {\color{black}ojha@arch.ethz.ch}}}
	\thanks{V Sn\'{a}\v{s}el is with the Dept. of Computer Science, Technical University of Ostrava, Czech Republic, e-mail: vaclav.snasel@vsb.cz}
	\thanks{A Abraham is with Machine Intelligence Research Labs, Washington, USA, e-mail: ajith.abraham@ieee.org}
	\thanks{This work was supported by the IPROCOM Marie Curie initial training network, funded through the People Programme (Marie Curie Actions) of the European Union's Seventh Framework Programme FP7/2007-2013/ under REA Grant Agreement No. 316555.}
	\thanks{Manuscript received Month xx, yyyy; revised Month xx, yyyy.}}

\markboth{Journal of \LaTeX\ Class Files,~Vol.~xx, No.~xx, Month~yyyy}%
{Ojha V.K., \MakeLowercase{\textit{et al.}}: Multiobjective Programming for Type-2 Hierarchical Fuzzy Inference Trees }

\maketitle
\begin{abstract}
This paper proposes a design of hierarchical fuzzy inference tree (HFIT).  {\color{black}An} HFIT produces an optimum tree-like structure. {\color{black}Specifically, a natural hierarchical structure} that accommodates simplicity by combining several low-dimensional fuzzy inference systems (FISs). Such {\color{black}a natural hierarchical structure} provides {\color{black}a high degree of approximation accuracy}. The construction of HFIT takes place in two phases. {\color{black}Firstly}, a nondominated sorting based {\color{black}multiobjective genetic programming} (MOGP) is applied to obtain a simple tree structure {\color{black}(low model's complexity) with a high accuracy}. {\color{black}Secondly}, the differential evolution algorithm is applied to {\color{black}optimize the obtained tree's parameters.} 
{\color{black}In the obtained tree, each node has a different input's combination, where the evolutionary process governs the input's combination. Hence, HFIT nodes are heterogeneous in nature, which leads to a high diversity among the rules generated by the HFIT.} 
{\color{black}Additionally}, the HFIT provides an automatic feature selection because it uses MOGP for the {\color{black}tree's structural optimization that accept inputs only relevant to the knowledge contained in data}. The HFIT was studied in the context of both type-1 and type-2 FISs, and its performance was evaluated {\color{black}through} six application problems. Moreover, the proposed multiobjective HFIT was compared {\color{black}both theoretically and empirically} with recently proposed FISs methods from the literature{\color{black},} such as McIT2FIS, TSCIT2FNN, SIT2FNN, RIT2FNS-WB, eT2FIS, MRIT2NFS, IT2FNN-SVR, etc. From the obtained results, it was found that {\color{black}the HFIT} provided less complex and highly accurate models compared to the models produced by most of the other methods. {\color{black}Hence}, the proposed HFIT is an efficient and competitive alternative to the other FISs for {\color{black}function} approximation and {\color{black}feature} selection.                  
\end{abstract}

\IEEEpeerreviewmaketitle

\section{Introduction}
\label{sec:intro}
\IEEEPARstart{A}{} fuzzy inference system (FIS)---composed of a fuzzifier to fuzzify input information, an inference engine to infer information from a rule base (RB), and a defuzzifier to return {\color{black}crisp} information---solves a wide range of problems that are ambiguous, uncertain, inaccurate, and noisy. An RB of an FIS is a set of rules of the form IF-THEN, i.e., the antecedent and the consequent form. {\color{black}The} Takagi--Sugeno--Kang (TSK) is a widely used FIS model{\color{black}~\cite{takagi1985fuzzy}}. It {\color{black}embraces the} IF-THEN form, where the antecedent part {\color{black}consists} of type-1 fuzzy {\color{black}sets} (T1FS) and/or type-2 fuzzy {\color{black}sets} (T2FS), and the consequent part consists of real values or a linear/nonlinear function. 

{\color{black}Type-1} FIS (T1FIS) and type-2 FIS (T2FIS) {\color{black}differ} when it comes to the representation of the antecedent part and the consequent part of a rule, and {\color{black}T1FS} and T2FS differ in the definitions of their membership functions. Unlike the crisp output of a T1FS membership function (MF)~\cite{Zadeh1965338}, the output of a T2FS MF is fuzzy in nature~\cite{zadeh1975concept}. {\color{black}Such nature of the T2FS MFs is advantageous in processing uncertain information more effectively than {\color{black}with} T1FS MFs~\cite{karnik1999type}. Hence, a T2FIS can overcome the inability of a T1FIS to fully handle or accommodate the linguistic and numerical uncertainties associated with {\color{black}a} changing and dynamic environment~\cite{hagras2004hierarchical}.}

However, a T2FIS is computationally expensive because it has a larger number of parameters than a T1FIS, and it requires a type-reduction mechanism {\color{black}in} its defuzzification part. The interval T2FIS (IT2FIS) reduces the computational cost by employing a simplified T2FS, known as interval T2FS (IT2FS)~\cite{karnik1999type}. An IT2FS MF is bounded by a lower MF (LMF) and an upper MF (UMF), and the area between the LMF and UMF is called the \textit{footprint of uncertainty}~\cite{karnik1999type}. Then, a type-reducer {\color{black}reduces} IT2FS to {\color{black}interval-valued} T1FS. Subsequently, the output of IT2FIS is produced by averaging the intervals. 

{\color{black}The construction and tuning of the rules are among the vital tasks in the optimization of an FIS, where the rule's construction is met by combining the fuzzy sets and the rule's tuning is met by adjusting the MF's parameters and the consequent part's parameters. Such {\color{black}a} form of rule optimization is often achieved by mapping the rule's parameters onto a real-valued genetic vector, and it is known as the Michigan {\color{black}Approach}~\cite{michigan1982}. Similarly, the construction/optimization of the RB is met by the genetic selection of the rules at the RB. Such {\color{black}a} form of RB optimization is often achieved by mapping the rules onto a binary-valued genetic vector~\cite{ishibuchi1999hybrid}, and it is known as the Pittsburgh {\color{black}Approach}~\cite{pittsburgh1980}. }

However, {\color{black}FIS} optimization is not limited only to its mapping onto the genetic vector, but a structural/network-like implementation of FIS is often performed~\cite{jang1993anfis}. Additionally, {\color{black}TSK}-based hierarchical self-organizing learning dynamics {\color{black}have} also been proposed~\cite{wu2000dynamic}. Moreover, several researchers {\color{black}have} focused {\color{black}on} the FIS and neural network (NN) integration and its parameter optimization using various learning methods including gradient-decent and the metaheuristic algorithms~\cite{sharma2009hybrid,huang2015optimization,castillo2012optimization,fernandez2015revisiting}. {\color{black}The summaries of such optimization paradigms are described as follows:}

{\color{black}A} self-constructing neural fuzzy inference network (SONFIN), proposed by Juang et al.~\cite{juang1998online}, is a six layered network structure whose optimization begins with no rule and {\color{black}then} rules are incrementally added during the learning process. SONFIN uses a clustering method to partition the input space that governs the number of rules extracted from the data{\color{black}{\color{black}, then} the parameters ({\color{black}MF's} arguments) of the determined SONFIN structure are tuned by the backpropagation algorithm.} Later, in~\cite{juang2008self}, {\color{black}SONFIN's} concept was extended for the construction of T2FIS, where a self-evolving IT2FIS (SEIT2FNN) that implements a TSK-type FIS model was proposed, and the parameters of the evolved structure were tuned by using the Kalman-filtering algorithm. Additionally, a simplified type-reduction process for SEIT2FNN was proposed in~\cite{juang2013reduced}. Like SONFIN, in~\cite{kasabov2002denfis}, a TSK-type FIS model{\color{black},} called {\color{black}a} dynamic evolving neural-fuzzy inference system (DENFIS){\color{black},} was proposed, which evolved incrementally by choosing active rules from a set of rules and employed an evolving clustering method to partition the input space and the least-square estimator to optimize its parameters.

To overcome some limitations of the self-organizing fuzzy NN paradigm, Tung et al.~\cite{tung2011safin} proposed a self-adaptive fuzzy inference network (SaFIN) that applied a categorical learning induced partitioning algorithm to eliminate two  limitations: 1) {\color{black}the} need {\color{black}for} predefined numbers of fuzzy clusters {\color{black}and} 2) the stability--plasticity trade-off that addresses the difficulty in finding a balance between {\color{black}past} knowledge and {\color{black}current} knowledge during {\color{black}the} learning process. SaFIN also employed a rule consistency {\color{black}checking} mechanism to avoid inconsistent RB construction. Additionally, the Levenberg-Marquardt method was applied for {\color{black}RB's} parameters tuning. In~\cite{lin2013mutually}, to improve the efficiency of IT2FIS, a mutually recurrent interval type-2 neural fuzzy system (MRIT2NFS) was proposed which used weighted feedback loops in the antecedent parts of the formed rules and applied gradient-decent learning and {\color{black}a} Kalman-filter algorithm to tune the recurrent weights and the rules' parameters, respectively. In~\cite{lin2014tsk}, a self-evolving T2FIS model was proposed that employed {\color{black}a} compensatory operator in the type-2 inference mechanism and a variable-expansive Kalman-filter algorithm for {\color{black}parameter} tuning. 

Further, a simplified interval type-2 fuzzy NN with a simplified type-reduction process (SIT2FIS)  was proposed in~\cite{lin2014simplified}, and a growing online self-learning IT2FIS that used the dynamics of {\color{black}a} growing Gaussian mixture model was proposed in~\cite{bouchachia2014gt2fc}. Recently, in~\cite{das2015evolving}, a meta-cognitive interval type-2 neuro FIS (McIT2FIS) was proposed, which {\color{black}employs} a self-regulatory meta-cognitive system that extracts the knowledge contained in minimal samples by accepting or discarding data samples based on {\color{black}sample's} contribution to knowledge. For the {\color{black}parameters tuning}, McIT2FIS employed {\color{black}the} Kalman-filtering algorithm.

{\color{black}However, the self-organizing fuzzy NN paradigm discussed above has to employ a clustering method to partition the input space during the FIS {\color{black}structure's} design. Contrary to this, a hierarchical FIS (HFIS) constructs an FIS by using a hierarchical arrangement of several low-dimensional fuzzy subsystems~\cite{raju1991hierarchical}.} Initially, the input variables selection, the levels of hierarchy, and the number of parameters was fully up to the experts to determine. {\color{black}Moreover}, HFIS design overcomes the \textit{curse of dimensionality}~\cite{brown1995high}, and it possesses a universal approximation ability~\cite{wang1999analysis,delgado2001hierarchical,lee2003modeling,zeng2005approximation}. 

Torra et al.~\cite{torra2002review} summarized the contributions where the expert's role in the HFIS design process was minimized/eliminated. For example, in~\cite{joo2005class}, HFIS was realized as {\color{black}a} feedforward network like structure in which the output of the previous {\color{black}layer's} subsystem was only fed to the consequent part of the next layer{\color{black},} and so on. Similarly, in~\cite{fernandez2009hierarchical}, a two-layered HFIS was developed, where, for each layer, the knowledge bases (KB) were generated by linguistics rule generation method and the {\color{black}KB rules} were selected by genetic algorithm (GA). In~\cite{hwang2014adaptive}, an adaptive fuzzy hierarchical sliding-mode control method was proposed{\color{black},} which was an arrangement of many subsystems, and the top layer accommodated all the {\color{black}subsystems'} outputs. Moreover, in~\cite{chen2007automatic}, to optimize the structure of a hierarchical arrangement of low-dimensional TSK-type {\color{black}FISs}, probabilistic incremental program evolution~\cite{salustowicz1997probabilistic} was employed. Similarly, the importance of the hierarchical arrangements of the low-dimensional T2FSs is explained in~\cite{hagras2004hierarchical,mohammadzadeh2014two}.

For {\color{black}FIS} models that have {\color{black}a} structural representation (e.g., self-organizing fuzzy NN and HFIS models), {\color{black}multiobjective} optimization is inherent since {\color{black}accuracy} maximization and complexity minimization are two desirable objectives~\cite{ishibuchi2007multiobjective}. Hence, to make trade-offs between interpretability and accuracy, or{\color{black},} in other words, to make trade-offs between approximation error minimization and complexity minimization, {\color{black}a} multiobjective orientation of {\color{black}FIS} optimization can be used~\cite{ishibuchi1997single,alcala2007multi,cordon2011historical}. Complexity minimization can be defined in many ways{\color{black},} such as {\color{black}a} reduced number of rules, reduced number of parameters, etc.~\cite{guillaume2001designing,cordon2011historical}. 

Since a single solution may not satisfy both objectives simultaneously, a Pareto-based multiobjective optimization algorithm can be used in FIS optimization{\color{black}, the scope of which} spans from the rule selection, to rule mining, rule learning, etc.~\cite{ishibuchi2007analysis,gacto2009adaptation,carmona2010nmeef,cara2013multiobjective}. Similarly, in~\cite{wang2005multi,munoz2008automatic,alcala2009multiobjective,antonelli2011learning}, simultaneous learning of KB was proposed, which included feature selection, rule complexity minimization together with approximation error minimization, etc. 

Moreover, in~\cite{antonelli2012genetic}, a co-evolutionary approach that {\color{black}aimed at}  combining {\color{black}a} multiobjective approach with {\color{black}a} single objective approach was presented where, at first, a multiobjective GA determined a {\color{black}Pareto-optimal} solution by finding a trade-off between accuracy and rule complexity. Then, a single objective GA was applied to reduce training instances. Such {\color{black}a} process was then repeated until a satisfactory solution was obtained. A summary of research works focused on multiobjective optimization of FIS is provided in~\cite{fazzolari2013review}.

{\color{black}In conclusion, the following are the necessary practices for an FIS model design: 1) input space partitioning; 2) {\color{black}rule formation}; 3) {\color{black}rule tuning}; 4) FIS structural representation; 5) improving accuracy and minimizing {\color{black}a model's} complexity. Therefore, in this work, a multiobjective optimization of HFIS, called {\color{black}a} hierarchical fuzzy inference tree (HFIT){\color{black},} was proposed.} 

Unlike the self-organizing paradigm that has a network-like structure and uses {\color{black}a} clustering algorithm for partitioning of {\color{black}input} space, the proposed HFIT constructs a tree-like structure and uses the dynamics of the evolutionary algorithm for partitioning {\color{black}input} space~\cite{poli2008field}. The HFIT is {\color{black}analogous} to {\color{black}a} multi-layered network and {\color{black}automatically partitions} input space during the structure optimization phase, i.e., during the tree construction phase. The parameter tuning of the HFIT was performed by the differential evolution (DE) algorithm~\cite{qin2009differential}, {\color{black}which is} a metaheuristic algorithm inspired by the dynamics of the evolutionary process. {\color{black}{\color{black}Metaheuristic} algorithms, being independent of the problems, solve complex optimization problems. Hence, {\color{black}they are} useful in finding the appropriate parameter values for an FIS~\cite{castillo2012optimization}. }

In this work, the proposed HFIT implements {\color{black}a} TSK-type FIS for both T1FIS and T2FIS, and {\color{black}HFIT} was studied under {\color{black}both} single objective and multiobjective optimization orientations. Hence, a total {\color{black}of} four versions of HFIT algorithms were proposed: type-1 single objective HFIT (T1HFIT$^{\text{S}}$), type-1 multiobjective objective HFIT (T1HFIT$^{\text{M}}$), type-2 single objective HFIT (T2HFIT$^{\text{S}}$), and type-2 multiobjective objective HFIT (T2HFIT$^{\text{M}}$). In the construction of type-2 HFITs, the type-reduction algorithm {\color{black}of the} {\color{black}Karnik-Mendel} method described in~\cite{karnik1999type} was used with an improvement {\color{black}in} its termination criteria. {\color{black}{\color{black}In} summary, the following are the main and novel contributions of this work.}

{\color{black} 
\begin{enumerate}
    \item The proposed hierarchical tree-like design (HFIT) forms a natural hierarchical structure by combining several low-dimensional fuzzy subsystems.
    \item MOGP driven optimization provided a trade-off between model's accuracy and complexity. Moreover, in the obtained tree, each node has a different input's combination, where the MOGP governs the input's combination. Hence, HFIT nodes are heterogeneous in nature, which leads to a high diversity among the rules generated by the HFIT. Such a diverse rule generation methods is a distinguished aspect of the proposed HFIT.
    \item A comprehensive theoretical study of HFIT shows that when it comes to the partitioning of input space, membership function design, and even rule formation, it has advantages over network-like layered architecture models, which have to use clustering methods when they do input space partitioning. Clustering methods generate overlapping MFs in fuzzy sets, whereas HFIT's MOGP driven MFs selection avoid such a overlapping of MFs.    
    \item Unlike many models in the literature, HFIT performed an inclusive automatic feature selection, which led to the simplification of the RB in fuzzy subsystems and incorporated only relevant knowledge contained in the dataset into HFIT's structural representation. 
    \item A comprehensive performance comparison of the proposed four versions of the HFIT algorithms both in theoretical and empirical sense with the recently proposed FIS algorithms found in the literature suggests that HFIT design offers a high approximation ability with simple model complexity.
\end{enumerate}}

{\color{black}The structure} of this article is as follows{\color{black}:} Section~\ref{sec_prilimi} provides an introduction to T1FIS and T2FIS{\color{black};} Section~\ref{sec_hfit} describes the proposed multiobjective strategy for developing HFIT {\color{black}and its parameter optimization; Section~\ref{sec_evaluation_the} provides a comprehensive theoretical evaluation of HFIT{\color{black};} Section~\ref{sec_evaluation_emp}} provides a detailed description of parameter setting and {\color{black} a comprehensive empirical evaluation} the proposed HFIT compared with the  algorithms reported in the literature{\color{black};} finally, the obtained results are discussed in Section~\ref{sec_disc} followed by a concise conclusion in Section~\ref{sec_con}.

\section{TSK Fuzzy Inference Systems}
\label{sec_prilimi}
\subsection{Type-1 Fuzzy Inference Systems}
A TSK-type FIS is governed by the {IF--THEN} rules of the form~\cite{takagi1985fuzzy}:
\begin{equation}
\label{eq_type1_rules}
R_i: \text{IF } x_1 \text{ is } {A}_{i1} \text{ AND } \ldots \text{ AND } x_{d^i} \text{ is } {A}_{id^i} \text{ THEN } y \text{ is } B_i
\end{equation}
where $ R_i $ is the $ i $-th rule in an FIS, $ A_{i1},\ldots, {A}_{id^i}$ are the T1FSs, $ B_i $ is a function of an input vector $ \text{\textbf{x}} = \langle  x_1,x_2,\ldots,x_{d^i}\rangle $ that returns a crisp output $ y $, and $ d^i $ is the total number of the inputs presented {\color{black}to} the $ i $-th rule. Note that the number of inputs may vary from rule-to-rule. Hence, the dimension of inputs {\color{black}in} a rule is denoted as $ d^i $. In TSK, the function $ B_i $ is usually expressed as: 
\begin{equation}
\label{eq_type1_consequent}
B_i = c_i^0 + \sum\limits_{j=1}^{d^i} c^j_i x_j
\end{equation}
where $ c^j_i$ for $ j $ = $ 0 $  to  $d^i$ is the free parameters {\color{black}in} the consequent part of a rule. 
The defuzzified crisp output of FIS is computed as follows{\color{black}:} First, the inference engine fires {\color{black}up} the {\color{black}RB rules}. The firing strength $ f_i $ of the $ i $-th rule is computed as:
\begin{equation}
\label{eq_type1_firing_strength}
f_i =  \prod\limits_{j = 1}^{d^i} \mu_{A_{ij}} (x_j) 
\end{equation}
where $ \mu_{A_{ij}} $ is the value of $ j $-th T1FS MF at the $ i $-th rule. Then, the defuzzified output $ y $ of an FIS is computed as: 
\begin{equation}
\label{eq_type1_out}
y = \frac{\sum_{i=1}^M B_i f_i}{\sum_{i=1}^M f_i}
\end{equation}
where $ M $ is the total rules in the RB. In this work, as shown in~Fig.~\ref{fig_MF_T1}, the T1FS $ A $ was of the form:
\begin{equation}
\color{blue}
\label{eq_type1_MF}
\mu_A(x) = \frac{1}{1 + \left(\frac{x-m}{\sigma} \right)^2}
\end{equation}
where $ m $ and $ \sigma $ are the center and the width of MF $ \mu_A(x) $, respectively.


\subsection{Type-2 Fuzzy Inference Systems}
\label{sec_fls2}
A T2FS $\tilde{A}$ is characterized by a {\color{black}3-dimensional (3-D)} MF~\cite{mendel2013km}. The three axes of T2FS are defined as follows. The x-axis is called {\color{black}the} primary variable, the y-axis is called {\color{black}the} secondary variable (or primary MF, which is denoted by $ u $), and the z-axis is called the MF value (or secondary MF value), which is denoted by $ \mu $. Hence, in a universal set $ X $, a T2FS $\tilde{A}$ has the form:
\begin{equation}
\label{eq_type2MF}
\tilde{A} = \left\lbrace \left( \left(x,u\right), \mu_{\tilde{A}} \left(x,u\right)  \right) | \forall x \in X, \forall u \in [0,1] \right\rbrace
\end{equation}
where the MF value $ \mu $ has a 2-dimensional support called the \textit{footprint of uncertainty} of $ \tilde{A} $, which is bounded by an LMF $ \underline{\mu}_{\tilde{A}}(x) $ and a UMF $ \bar{\mu}_{\tilde{A}} (x)$ (Fig.~\ref{fig_MF_T2}). 
A Gaussian function{\color{black},} with {\color{black}an} uncertain mean within $ [m_1, m_2] $ and standard deviation $ \sigma ${\color{black},} is an IT2FS MF (Fig.~\ref{fig_MF_T2}), which is written as: 
\begin{equation}
\label{eq_GaussT2MF}
\color{blue}
\mu_{\tilde{A}} (x,m,\sigma) = \exp \left( -\frac{1}{2} \left(\frac{x - m}{\sigma}\right)^2\right),\hspace*{2em} m \in [m_1,m_2].
\end{equation}
In this work, the LMF was defined as~\cite{karnik1999type}:
\begin{equation}
\label{eq_loMF}
\underline{\mu}_{\tilde{A}} (x) = \left\lbrace
\begin{array}{ll}
\mu_{\tilde{A}} (x,m_2,\sigma),\hspace*{2em} & x \le (m_1 + m_2)/2\\
\mu_{\tilde{A}} (x,m_1,\sigma),\hspace*{2em} & x > (m_1 + m_2)/2
\end{array}
\right.
\end{equation}
and the UMF was defined as~\cite{karnik1999type}:
\begin{equation}
\label{eq_upMF}
\bar{\mu}_{\tilde{A}} (x) = \left\lbrace
\begin{array}{ll}
\mu_{\tilde{A}} (x,m_1,\sigma),\hspace*{2em} & x < m_1\\
1,\hspace*{2em} &  m_1 \le x \le m_2 \\
\mu_{\tilde{A}} (x,m_2,\sigma),\hspace*{2em}  & x > m_2
\end{array}.
\right.
\end{equation}
\begin{figure}
	\centering
	\subfigure[Type-1 Fuzzy MF]
	{
		\includegraphics[width=0.45\columnwidth]{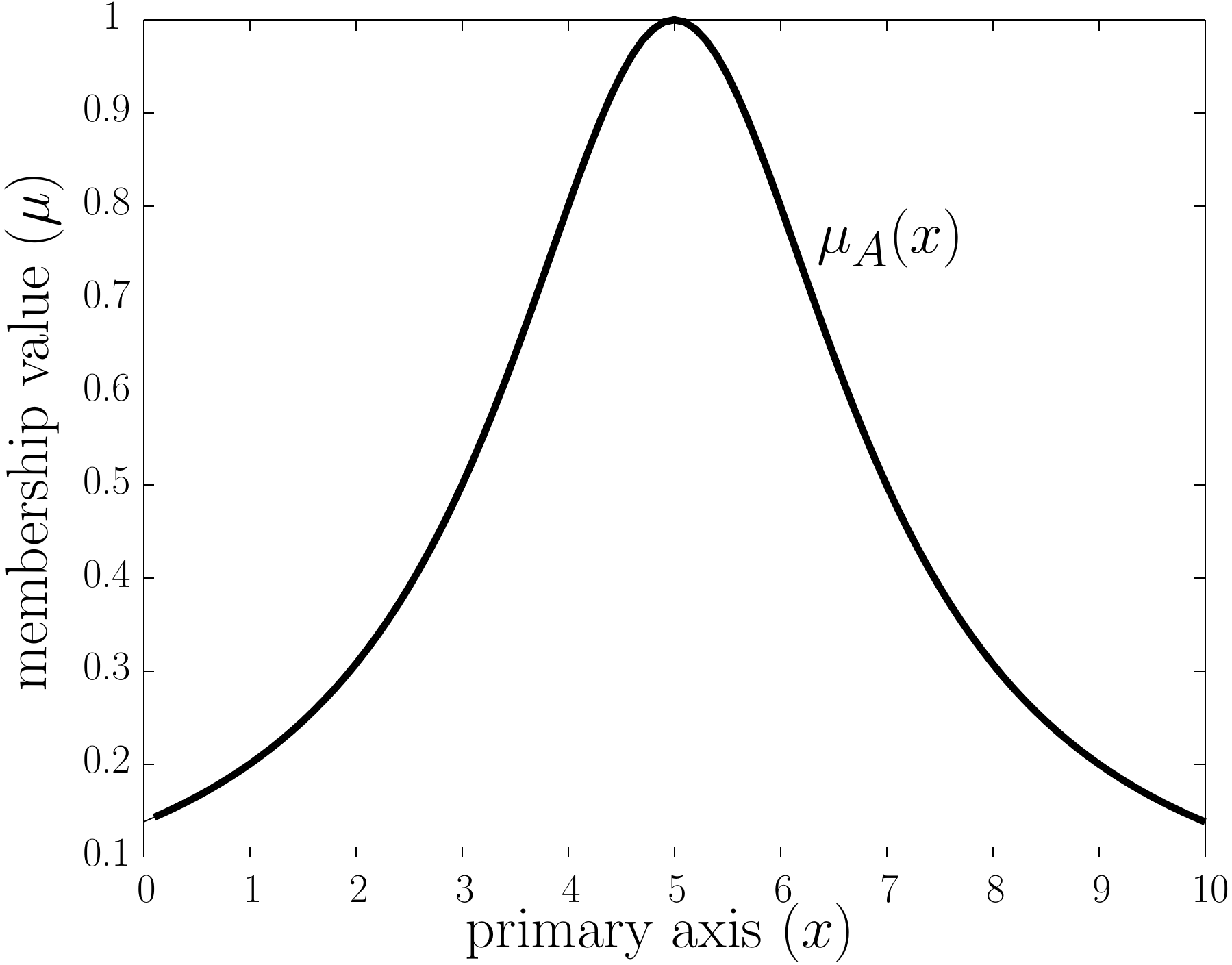}
		\label{fig_MF_T1}%
		
	}
	\subfigure[Type-2 Fuzzy MF]
	{
	\includegraphics[width=0.45\columnwidth]{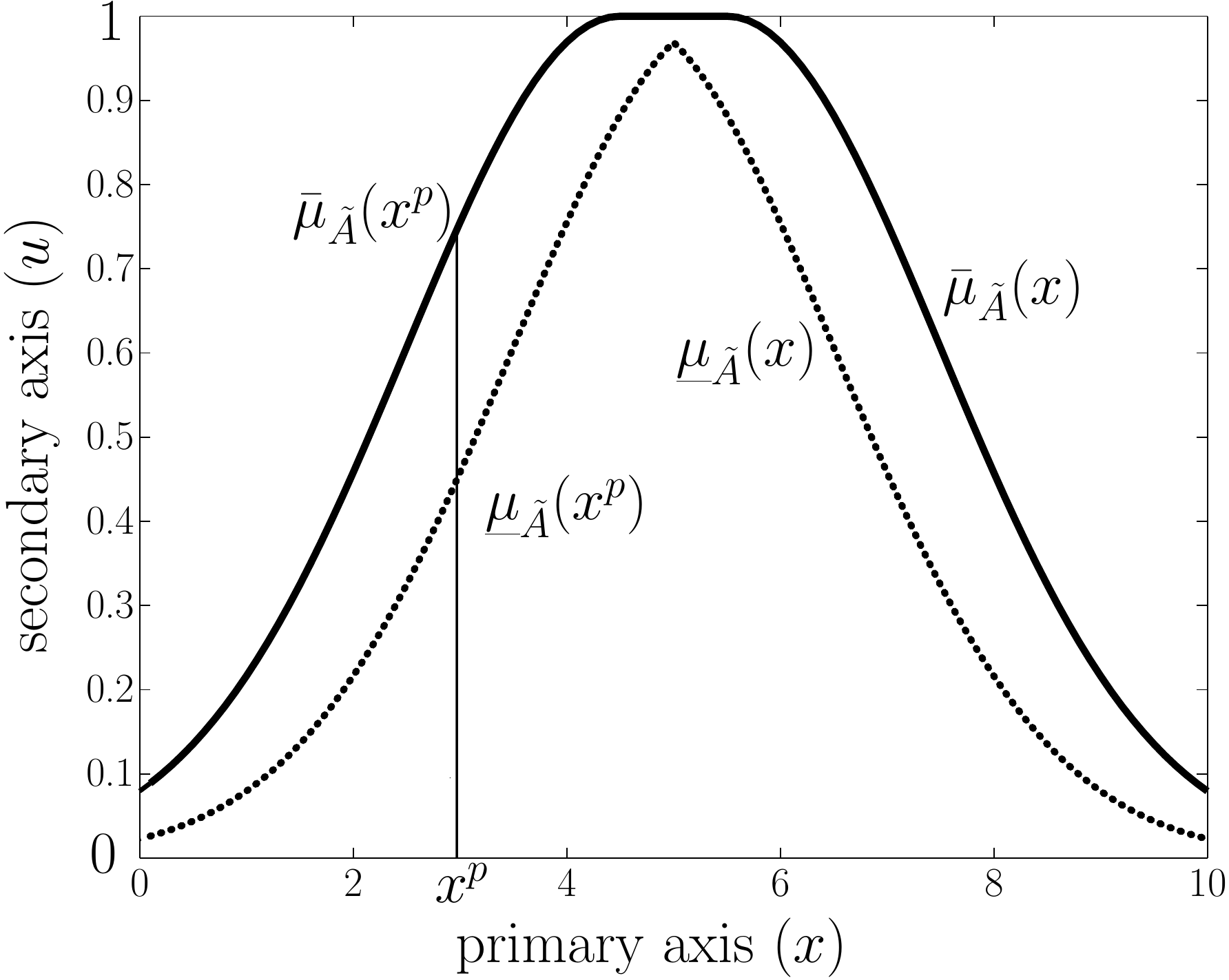}
	\label{fig_MF_T2}
	}
	\caption{Membership functions. (a) Type-1 MF~\eqref{eq_type1_MF} with mean $ m = 5.0 $ and $ \sigma = 2.0 $ (b) Type-2 Fuzzy MF with fixed $ \sigma = 2.0 $ and means $ m_1 = 4.5 $ and $ m_2 = 5.5 $. LMF $ \underline{\mu}_{\tilde{A}} (x) $  as per~\eqref{eq_loMF} is {\color{black}on the} dotted line and UMF $ \bar{\mu}_{\tilde{A}} (x) $ as per~\eqref{eq_upMF} is {\color{black}on the} solid line.}
\end{figure}
In Fig.~\ref{fig_MF_T2}, {\color{black}the} point $ x^p $ along the x-axis of 3-D IT2FS MF cuts the LMF and UMF along the y-axis, and the value of the IT2FS is considered to be along the  z-axis (not shown in Fig.~\ref{fig_MF_T2}) are  $ \bar{\mu}_{\tilde{A}} (x^p) $   and $ \underline{\mu}_{\tilde{A}} (x^p) $. Considering IT2FS MFs, $ i $-th IF--THEN rule of type-2 TSK-FIS for an input vector $ \text{\textbf{x}} = \langle x_1,x_2,\ldots,x_{d^i}\rangle $ takes the following form:
\begin{equation}
\label{eq_type2_rules}
R_i: \text{IF } x_1 \text{ is } \tilde{A}_{i1} \text{ AND } \ldots \text{ AND } x_{d^i} \text{ is } \tilde{A}_{id^i} \text{ THEN } y \text{ is } \tilde{B}_i
\end{equation}
where $ \tilde{A}_{i1},\ldots, \tilde{A}_{id^i}$ are the T2FSs, $ \tilde{B}_i $ is  a function of $ \text{\textbf{x}} $ that returns a pair $[\underline{b}_i, \bar{b}_i]$ called {\color{black}the} left and right weights of the consequent part of the $ i $-th rule. In TSK, $ \tilde{B}_i $ is usually written as:
\begin{equation}
\label{eq_type2_consequent}
\tilde{B}_i  = [c^0_i-s^0_i,{c}^0_i+s^0_i]  + \sum\limits_{j=1}^{d^i} [c^j_i-s^j_i,{c}^j_i+s^j_i]x_j
\end{equation}
where $ c^j_i $ for $ j$ = $ 0 $  to  $d^i$ is the free parameter {\color{black}in} the consequent part of a rule and $ s^i_j $ for $ j$ = $ 0 $  to  $d^i$ {\color{black}are} the deviation factors of the free parameters.  The firing strength of IT2FS $ F_i = [\underline{f}_i, \bar{f}_i]$ is computed as:
\begin{equation}
\label{eq_fire}
\underline{f}_i = \prod\limits_{j=1}^{d^i} \underline{\mu}_{\tilde{A}_{ij}} (x_j) \text{\hspace*{1em} and \hspace*{1em}}  \bar{f}_i = \prod\limits_{j=1}^{d^i} \bar{\mu}_{\tilde{A}_{ij}} (x_j)
\end{equation}

At this stage, inference engine fires {\color{black}up} the rule and the type-reducer reduces the IT2FS to T1FS. In this work, {\color{black}the} \textit{center of set} type-reducer $ y_{cos} $, prescribed  in~\cite{karnik1999type}, was used:
\begin{equation}
\label{eq_cos}
y_{cos} = \underset{f^i \in F^i, \,\, b^i \in \tilde{B}^i}{\bigcup} \frac{\sum_{i = 1}^M f^i b^i}{\sum_{i = 1}^M f^i} = [y_l, y_r]
\end{equation}    
where $ y_l $ and $ y_r $ are the left and the right end of the interval. For the {\color{black}ascending} order of $ \underline{b}^i $ and $ \bar{b}^i$ , $ y_l $ and $ y_r $ are computed as:
\begin{equation}
\label{eq_yl}
y_l = \frac{\sum_{i = 1}^L \bar{f}^i \underline{b}^i  + \sum_{i = L+1}^{M}\underline{f}^i \underline{b}^i}{\sum_{i = 1}^L\bar{f}^i + \sum_{i = L+1}^{M}\underline{f}^i}
\end{equation}
\begin{equation}
\label{eq_yr}
y_r = \frac{\sum_{i = 1}^R \underline{f}^i \bar{b}^i  + \sum_{i = R+1}^{M}\bar{f}^i \bar{b}^i}{\sum_{i = 1}^R\underline{f}^i + \sum_{i = R+1}^{M}\bar{f}^i}
\end{equation}
where $ L $ and $ R $ are the switch point, determined by
$$ \underline{b}^L \le y_l \le \underline{b}^{L+1}\hspace*{1em} \text{ and }\hspace*{1em} \bar{b}^R \le y_r \le \bar{b}^{R+1}, $$
respectively. The defuzzified crisp output is then computed as:
\begin{equation}
\label{eq_fist2_out}
y = \frac{y_l + y_r}{2}.
\end{equation} 
\section{Multiobjective Optimization of Hierarchical Fuzzy Inference Trees}
\label{sec_hfit}
\subsection{Hierarchical Tree Formation}
\label{sec_hfit_tree}
A hierarchical fuzzy inference tree (HFIT) is a tree-based system. Its hierarchical structure is analogous to a multilayer feedforward NN, where the nodes  (the low-dimensional FISs) are connected using weighted links. The concept of forming {\color{black}a} hierarchical fuzzy inference tree is inherited from the flexible neural tree proposed by Chen et al.~\cite{chen2005time}, which has two learning {\color{black}phases}. First, in {\color{black}the} \emph{tree construction} phase, an evolutionary algorithm is employed to construct/optimize a tree-like structure. Second, in {\color{black}the} \emph{parameter tuning} phase, a genotype representing the underlying parameters of the tree structure is optimized by using parameter optimization algorithms. 

To create an optimum tree based model, {\color{black}firstly}, a population of randomly generated trees is formed. Once a satisfactory tree structure (a tree with {\color{black}a} small approximation error and low complexity) is obtained using an evolutionary algorithm, the \emph{parameter tuning} phase optimizes its parameters. The phases are repeated until a satisfactory solution is obtained. {\color{black} Fig.~\ref{fig_HFIT_dia}} is a clear representation of {\color{black}HFIT's two-phase construction approach}.
\begin{figure}
	\centering
	\includegraphics[width=0.5\columnwidth]{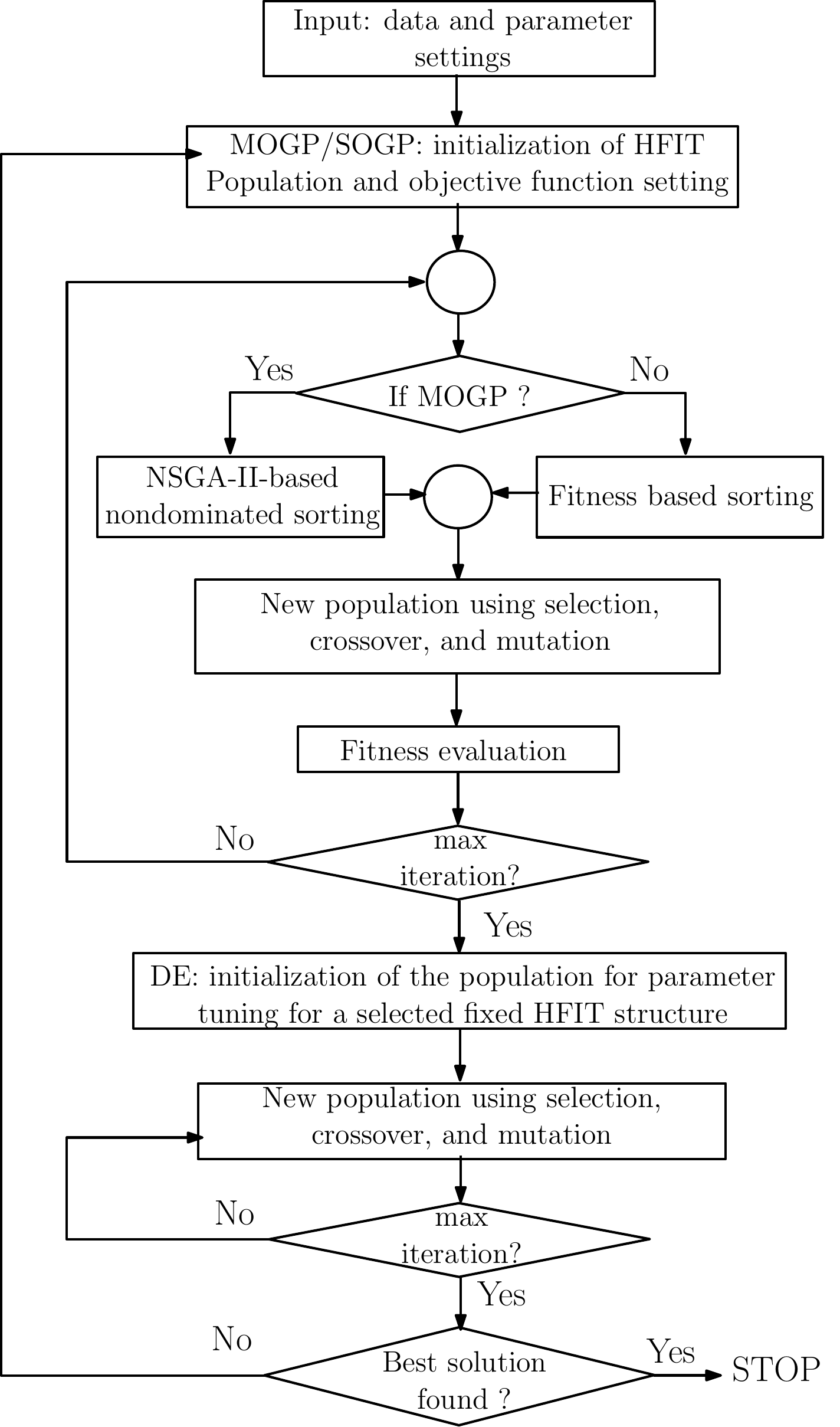}
	\caption{Two-phase construction of {\color{black}a} hierarchical fuzzy inference tree.}
	\label{fig_HFIT_dia}
\end{figure}

\subsection{Tree Encoding}
\begin{figure}
	\centering
	\subfigure[{\color{black}Two-stage} hierarchical tree.]
	{
		\includegraphics[width=0.45\columnwidth]{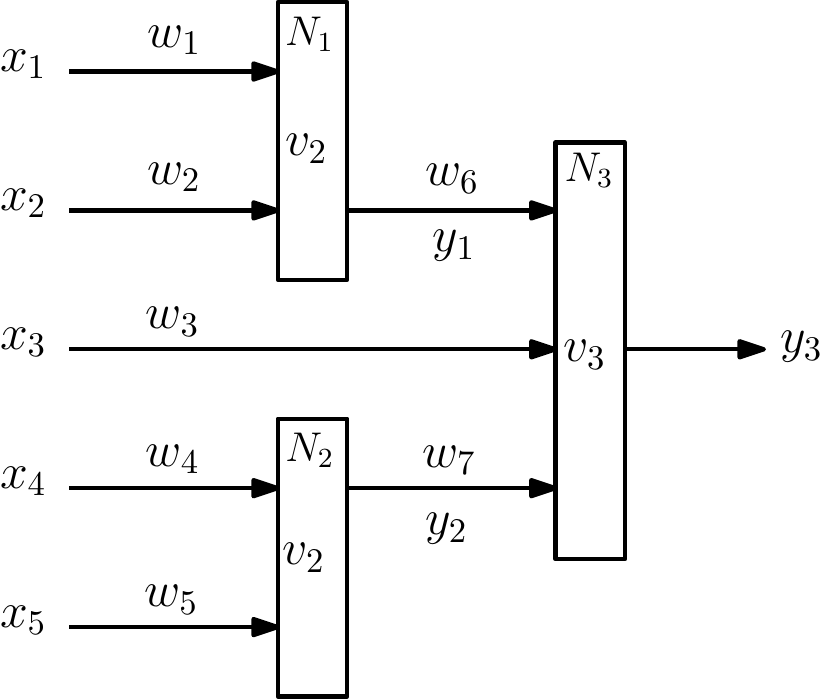}
		\label{fig_HFIS_tree}%
	}
	\subfigure[{\color{black}Node structure.}]
	{
		\includegraphics[width=0.45\columnwidth]{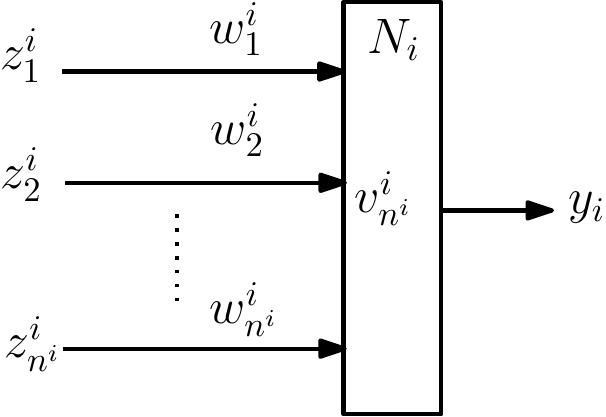}
		\label{fig_HFIS_node}%
	}
	\caption{Hierarchical fuzzy inference system. (a) Complete tree with three nodes $ N_1 $, $ N_2 $, and $ N_3 $ and with inputs $ x_1 $, $ x_2 $, $ x_3 $, $ x_4 $, and $ x_5$. (b) Illustration of the $ i $-th node $ N_i $ that has $ n^i $ inputs $ z^i_j \in [x_1, \ldots,x_n]$ for $ j = 1 \text{ to } n^i$ and output $ y_i $.}
	\label{fig_HFIS}
\end{figure}

An HFIT $ G $ is a collection of nodes $ V $ and terminal nodes $ T $:
{\color{black}\begin{equation}
	\label{eq-fnt}
	G = V \cup T = \left\lbrace v_2,v_3,\ldots,v_{tn} \right\rbrace  \cup \left\lbrace x_1,x_2,\ldots,x_d \right\rbrace 
\end{equation}}
where $ v_j \left( j = 2,3,\ldots,{tn} \right)  $ denotes non-leaf instruction and has $ 2 \le j \le tn $ arguments. The leaf node's instruction $ x_1,x_2,\ldots,x_d $ takes no argument and represents the input variable/instruction. A typical HFIT is shown in Fig.~\ref{fig_HFIS_tree}; whereas, Fig.~\ref{fig_HFIS_node} {\color{black}illustrates an HFIT's} node $ N_i $ that takes $ n^i $ inputs. The inputs $ z^i_j \in \{x_1,x_2,\ldots,x_d\}$ for $ j = 1 \text{ to } n^i$ to {\color{black}a} node $ N_i $ is either from the input layer or from other nodes in HFIT. Each node in an HFIT receives a weighted input $ x_i w_i  $, where $ w_i $ is the weight. {\color{black}In this work, however, the weights in HFIT were set to 1 because the objective of this work was also to reduce the complexity of the produced tree along with approximation error. Setting weights to 1 also allow raw input to be fed to the fuzzy sets.}

\subsection{Rule Formation at the Nodes}
\label{sec_hfit_Rule}
\subsubsection{Rules for Type-1 FIS Node}
Each node in an HFIT is an FIS of either type-1 or type-2. Hence, the rules at a node were created as follows{\color{black}:} Considering a reference to the node $ N_1 $ from Fig.~\ref{fig_HFIS_tree} that has two arguments/inputs $ x_1 $ and $ x_2 $ and assuming that each input $ x_1$ and $ x_2 $ has two T1FSs $ A^1_{11}, A^1_{12} $ and $ A^1_{21}, A^1_{22} $, respectively, the rules for T1FIS {\color{black}are} generated as:  
\begin{equation*}
\begin{array}{l}
R^1_{ij}: \text{IF } x_1 \text{ is } {A}^1_{1i} \text{ AND } x_2 \text{ is } {A}^1_{2j} \text{ THEN }
y^1_{ij} =  c^0_{ij} + c^1_{ij} x_1 + c^2_{ij} x_2,\\ \text{for } i = 1,2 \text{ and } j = 1,2.
\end{array}
\end{equation*}

{\color{black}The consequent part $ B^1_{ij} $ of the rules at the node $ N_1 $ is computed by using~\eqref{eq_type1_consequent}. Finally, the output $ y_1 $ of node $ N_1 $ is computed as:
\begin{equation}
y_1 = \frac{\sum_{i=1}^{2}\sum_{j=1}^{2} f^1_{ij} B^1_{ij}}{\sum_{i=1}^{2}\sum_{j=1}^{2} f^1_{ij} }
\end{equation}
where the firing strength $ f^1_{ij} $ is computed as:
\begin{equation}
f^1_{ij} = \mu_{A^1_{1i}}(x_1) \mu_{A^1_{2j}}(x_2),\hspace*{2em} \text{for } i = 1,2 \text{ and } j = 1,2.
\end{equation}

Similar to {\color{black}node} $ N_1 $, {\color{black}node} $ N_2 $ has two inputs and{\color{black},} if each input at node $ N_2 $ is partitioned into two T1FSs{\color{black}, then} the output $ y_2 $ of node $ N_2 $ is computed in {\color{black}a} similar way {\color{black}to how} the output of the node $ N_1 $ is computed. 

The output $ y_3 $ of the HFIT shown in Fig.~\ref{fig_HFIS_tree} is computed from {\color{black}node} $ N_3 ${\color{black}, which} revives inputs $ y_1 $ and $ y_2 $ and $ x_3 $, where $ y_1 $ and $ y_2 $ are the outputs of {\color{black}nodes} $ N_1 $ and $ N_2 $, respectively. Therefore, the rules at node $ N_3 $, considering each input $ y_1$, $ y_2 $, and $ x_3 $ has two T1FSs $ A^3_{11}, A^3_{12} $, $ A^3_{21}, A^3_{22} $, and $ A^3_{31}, A^3_{32} $ respectively, is represented as:
\begin{equation*}
\begin{array}{l}
R^3_{ijk}: \text{IF } y_1 \text{ is } {A}^3_{1i} \text{ AND } y_2 \text{ is } {A}^3_{2j} \text{ AND } x_3 \text{ is } {A}^3_{3k} \text{ THEN }
y^3_{ijk} =  c^0_{ijk} + c^1_{ijk} y_1 + c^2_{ijk} y_2 + c^3_{ijk} x_3,\\ \text{for } i = 1,2 \text{ and } j = 1,2 \text{ and } k = 1,2.
\end{array}
\end{equation*}
Output $ y_3 $ of node $ N_3 $, which is also the output of the tree (Fig.~\ref{fig_HFIS_tree}){\color{black},} is computed as:
\begin{equation}
y_3 = \frac{\sum_{i=1}^{2}\sum_{j=1}^{2}\sum_{k=1}^{2} f^3_{ijk} B^3_{ijk}}{\sum_{i=1}^{2}\sum_{j=1}^{2} \sum_{j=1}^{2}f^3_{ijk} }
\end{equation}
where the consequent part $ B^3_{ij} $ is computed using~\eqref{eq_type1_consequent} and the firing strength $ f^3_{ijk} $ is computed as:
\begin{equation}
f^3_{ij} = \mu_{A^3_{1i}}(y_1) \mu_{A^3_{2j}}(y_2) \mu_{A^3_{3k}}(x_3),\hspace*{2em} \text{for } i = 1,2 \text{ and } j = 1,2 \text{ and } k = 1,2.
\end{equation}}

\subsubsection{Rules for Type-2 FIS Node}
If the nodes of the HFIT in Fig.~\ref{fig_HFIS_tree} are type-2 nodes{\color{black}, then}, assuming that {\color{black}node} $ N_1 $ has two T2FSs $ \tilde{A}^1_{11}, \tilde{A}^1_{12} $ and $ \tilde{A}^1_{21}, \tilde{A}^1_{22} $, respectively, the rules for T2FIS at {\color{black}node} $ N_1 $ {\color{black}are} generated as:

\begin{equation*}
\begin{array}{l}
R^1_{ij}: \text{IF } x_1 \text{ is } \tilde{A}^1_{1i} \text{ AND } x_2 \text{ is } \tilde{A}^1_{2j} \text{ THEN }
y^1_{ij} =  [c^0_{ij}-s^0_{ij}] + [c^1_{ij}-s^1_{ij}] x_1 + [c^2_{ij}-s^2_{ij}] x_2,\\
\text{for } i = 1,2 \text{ and } j = 1,2
\end{array}
\end{equation*}
and the lower and upper firing strengths $ \underline{f}^1_{ij} $ and $ \bar{f}^1_{ij} $ at {\color{black}node} $ N_1 $ {\color{black}are} computed as: 
\begin{equation}
\underline{f}^1_{ij} = \mu_{\tilde{A}^1_{1i}}(x_1) \mu_{\tilde{A}^1_{2j}}(x_2),\hspace*{2em} \text{for } i = 1,2 \text{ and } j = 1,2
\end{equation}
\begin{equation}
\bar{f}^1_{ij} = \mu_{\tilde{A}^1_{1i}}(x_1) \mu_{\tilde{A}^1_{2j}}(x_2),\hspace*{2em} \text{for } i = 1,2 \text{ and } j = 1,2.
\end{equation}

{\color{black}Then, the left and right weights $\underline{b}^1_{ij}$ and $\bar{b}^1_{ij}$  of the consequent part of the rules are produced by using~\eqref{eq_type2_consequent}. {\color{black}Thereafter}, the type-reduction of the node is performed as described in~\cite{karnik1999type}, where the left and right intervals $ y^1_l $ and $ y^1_r $ are computed as per~\eqref{eq_yl} and \eqref{eq_yr}. During {\color{black}type-reduction}~\cite{karnik1999type}, an early stopping mechanism was adopted to reduce computation time. Finally, {\color{black}output} $ y_1 $ of {\color{black}node} $ N_1 $ is computed as $ y_1 = (y^1_l + y^1_r )/2 $. 

The output computation at {\color{black}node} $ N_2 $ of the tree in Fig.~\ref{fig_HFIS_tree} is similar to that of the output computation of {\color{black}node} $ N_1 $ because{\color{black},} at {\color{black}node} $ N_2 ${\color{black},} there are two inputs and each of {\color{black}these are} partitioned into two T2FSs.  

The output of the type-2 HFIT shown in Fig.~\ref{fig_HFIS_tree} is computed from {\color{black}node} $ N_3 ${\color{black}, which receives} inputs $ y_1 $ and $ y_2 $ and $ x_3 $, where $ y_1 $ and $ y_2 $ are the outputs of {\color{black}nodes} $ N_1 $ and $ N_2 $, respectively. Therefore, the rules at node $ N_3 $, considering each input $ y_1$, $ y_2 $, and $ x_3 $ has two T2FSs $ \tilde{A}^3_{11}, \tilde{A}^3_{12} $, $ \tilde{A}^3_{21}, \tilde{A}^3_{22} $, and $ \tilde{A}^3_{31}, \tilde{A}^3_{32} $ respectively, {\color{black}are} represented as:
\begin{equation*}
\begin{array}{l}
R^3_{ijk}: \text{IF } y_1 \text{ is } \tilde{A}^3_{1i} \text{ AND } y_2 \text{ is } \tilde{A}^3_{2j} \text{ AND } x_3 \text{ is } \tilde{A}^3_{3k} \text{ THEN }\\
\hspace*{3em}y^3_{ijk} =  [c^0_{ijk}-s^0_{ijk}] + [c^1_{ijk}-s^1_{ijk}] y_1 + [c^2_{ijk}-s^2_{ijk}] y_2 + [c^3_{ijk}-s^3_{ijk}] x_3,\\ \text{for } i = 1,2 \text{ and } j = 1,2 \text{ and } k = 1,2.
\end{array}
\end{equation*}
The lower and upper firing strengths $ \underline{f}^3_{ijk} $ and $ \bar{f}^3_{ijk} $ at {\color{black}node} $ N_3 $ {\color{black}are} computed as: 
\begin{equation}
\underline{f}^3_{ijk} = \mu_{\tilde{A}^3_{1i}}(y_1) \mu_{\tilde{A}^3_{2j}}(y_2) \mu_{\tilde{A}^3_{3j}}(x_3), \text{ for } i = 1,2 \text{ and } j = 1,2 \text{ and } k = 1,2
\end{equation}
\begin{equation}
\bar{f}^3_{ijk} = \mu_{\tilde{A}^3_{1i}}(y_1) \mu_{\tilde{A}^3_{2j}}(y_2) \mu_{\tilde{A}^3_{3j}}(x_3), \text{ for } i = 1,2 \text{ and } j = 1,2 \text{ and } k = 1,2.
\end{equation}

After computing the firing strengths and the left and right weights $\underline{b}^3_{ij}$ and $\bar{b}^3_{ij}$ of the rules, the type-reduction at the node is performed by using~\eqref{eq_cos}, where the left and right intervals $ y^3_l $ and $ y^3_r $ are computed as per~\eqref{eq_yl} and \eqref{eq_yr}.  Output $ y_3 $ of {\color{black}node} $ N_3 $, which is also the output of the tree{\color{black},} is computed by averaging $ y^3_l $ and $ y^3_r $ as  $ y_3 = (y^3_l + y^3_r )/2 $.}

\subsection{Structure Tuning (Pareto-based Multiobjective Optimization)}
\label{sec_mo_Strategy}
{\color{black}Usually, a learning algorithm owns a single objective (approximation error minimization) that is often achieved by minimizing the root mean squared error (RMSE) on the learning data}:
\begin{equation}
\label{eq_rmse}
E =  \sqrt{ \frac{1}{N} \sum\limits_{i = 1}^{N} (d_i - y_i)^2}
\end{equation} 
where $ d $ and $ y $ are the desired and the model's outputs, respectively{\color{black},} and $ N $ is the number of data pairs in the training set. However, a single objective comes at the expense of {\color{black}a} model's complexity or the generalization ability on unseen data. The generalization ability broadly depends on the model's complexity {\color{black}(e.g., the number of parameters $ c(\text{\textbf{w}}) $ in the model)}~\cite{jin2005evolutionary}. The minimization of the approximation error $ E $ and the number of free parameters $ c(\text{\textbf{w}}) $ are conflicting objectives. {\color{black}Hence, a Pareto-based multiobjective optimization can be applied to obtained a Pareto set of nondominated solutions, in which no one objective function can be improved without a simultaneous detriment to at least one of the other objectives of the solution~\cite{deb2000fast}}

Therefore, an HFIT that offers the lowest approximation error and simplest structure is the {\color{black}most} desirable {\color{black}one}. To obtain such {\color{black}a} set of Pareto-optimal (nondominated) solutions, a nondominated sorting {\color{black} based MOGP} was applied. 

{{\color{black} The proposed MOGP acquires the nondominated sorting algorithm}~\cite{deb2000fast} for computing Pareto-optimal solutions from an initial population of fuzzy inference trees. The individuals in {\color{black}MOGP} {\color{black}were} sorted according to their dominance in population. Moreover, individuals were sorted according to the rank/Pareto-front/line. {\color{black}MOGP} is an elitist algorithm that allows the best individuals to propagate into {\color{black}the} next generation. Diversity in population was maintained by measuring the crowding distance among the individuals~\cite{deb2000fast}.

{\color{black}A} detailed description of {\color{black} MOGP algorithm} {\color{black}is} as follows{\color{black}:}
\subsubsection{Initial Population}
Two fitness {\color{black}measures were} considered: \emph{approximation error} minimization and \textit{parameter count} minimization. {\color{black}To simultaneously optimize these objectives} {\color{black}during} the \emph{structure-tuning} phase{\color{black}, an} initial population $\text{W}^0$ of randomly generated HFITs was formed and sorted according to their nondominance{\color{black}.} 

\subsubsection{Selection} In selection operation, a \emph{mating pool} $ \text{W}^{p} $ of $ size(\text{W}^0)/2 $ was obtained using \emph{binary tournament selection} that selects two candidates randomly at a time from a  population $ \text{W}^t $, and the best solution (according to its rank and crowding distance) is copied into the mating pool $ \text{W}^{p} $. This process is continued until the mating pool becomes full. 

\subsubsection{Generation} An offspring population $ \text{W}^{c}  $ was generated using the individuals of the mating pool $ \text{W}^{p} $. Two distinct individuals (parents) were randomly selected from the mating pool to create new individuals using the genetic operators crossover and mutation{\color{black}.} 

\subsubsection{Crossover} In crossover operation, {\color{black}randomly} selected sub-trees of two parent trees are swapped (Fig.~\ref{fig_tree_crossover}). The swapping includes the exchange of nodes. A detailed description of the crossover operation in genetic programming is available in~\cite{eiben2015ec,ojha2017ensemble}. The crossover operation is selected with {\color{black}the} crossover probability $ pc $. 

\subsubsection{Mutation} The mutation operators used in HFIT are as follows~\cite{eiben2015ec,ojha2017ensemble}:
{\small 
\begin{enumerate}
	\item[a)]  Replacing a randomly selected terminal $ x_i \in T$ with a newly generated terminal $ x_j \in T$ for $ j \ne i $.
	\item[b)]  Replacing all terminal nodes of an HFIT with a new set of terminal nodes derived from $T$.
	\item[c)]  Replacing a randomly selected FIS node $ N_i \in F$ with a newly generated FIS  node $ N_j \in F$ for $ j \ne i $.
	\item[d)]  Replacing a randomly selected terminal node $ x_i \in T$ with a newly created FIS node $ N_i \in F$.
	\item[e)]  Deleting a randomly selected terminal node $ x_i \in T$ or deleting a randomly selected FIS node $ N_i \in F$.  
\end{enumerate}
}
The mutation operation was selected with {\color{black}the} probability $ pm $, and the type of mutation operator (a or b or c or d or e) was chosen randomly during the mutation operation (Fig.~\ref{fig_tree_mutation}).
\subsubsection{Recombination} The offspring population $ \text{W}^{c} $ and the main population $ \text{W}^{t} $ were mixed together {\color{black}to make} a combined population $ \text{W}^{g} $. 
\subsubsection{Elitism} {\color{black}In this work, elitism was decided according to the rank (based on both RMSE and {\color{black}the} model's complexity) of the individuals (HFITs) in the population. Therefore, in this step, $ size(\text{W}^{c}) $ worst  (poorer rank)  individuals were weeded out from the combined population $ \text{W}^{g} $. In other words, $ size(\text{W}^{t}) $ best individuals are propagated {\color{black}into the} new generation $ t+1 $ as the main population $ \text{W}^{t+1} $. }
\begin{figure}
	\centering
	\subfigure[Crossover operation]
	{
		\includegraphics[width=0.4\columnwidth]{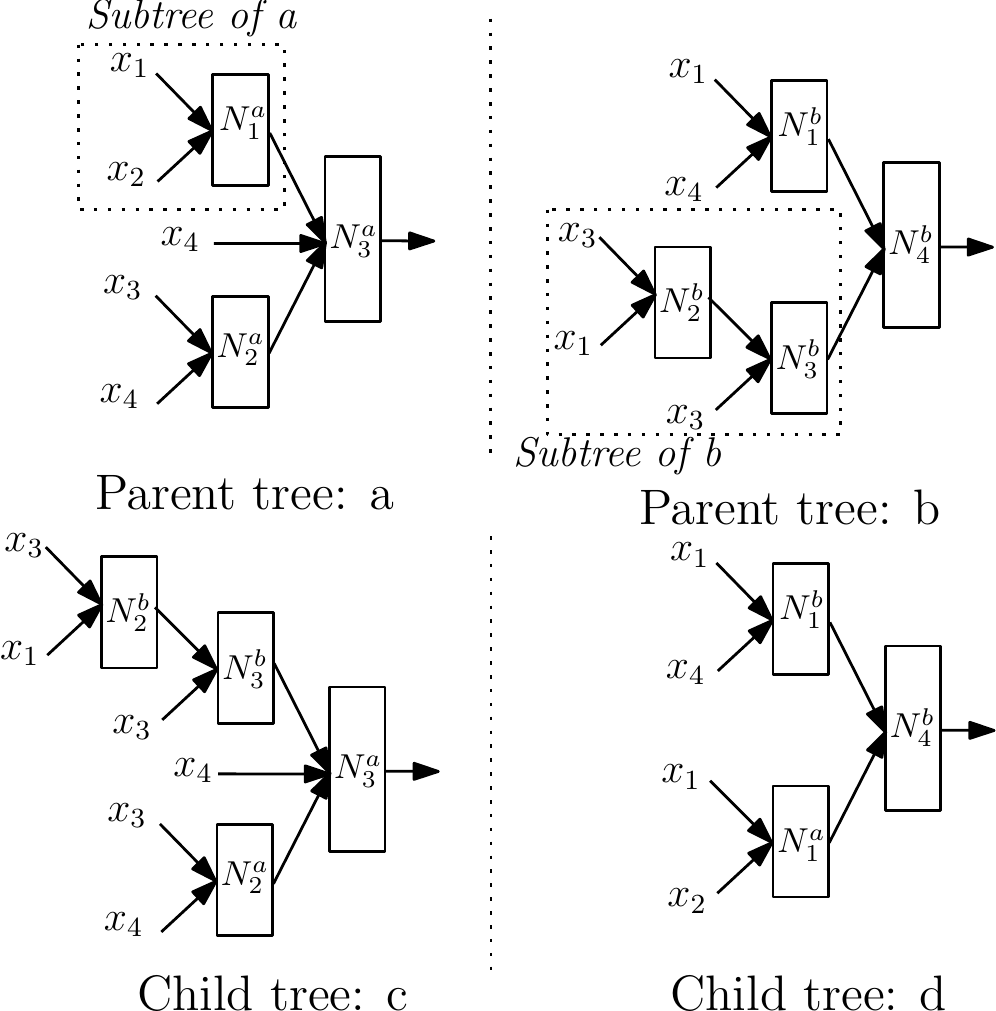}
		\label{fig_tree_crossover}%
		
	}
	\subfigure[Mutation operation]
	{
		\includegraphics[width=0.5\columnwidth]{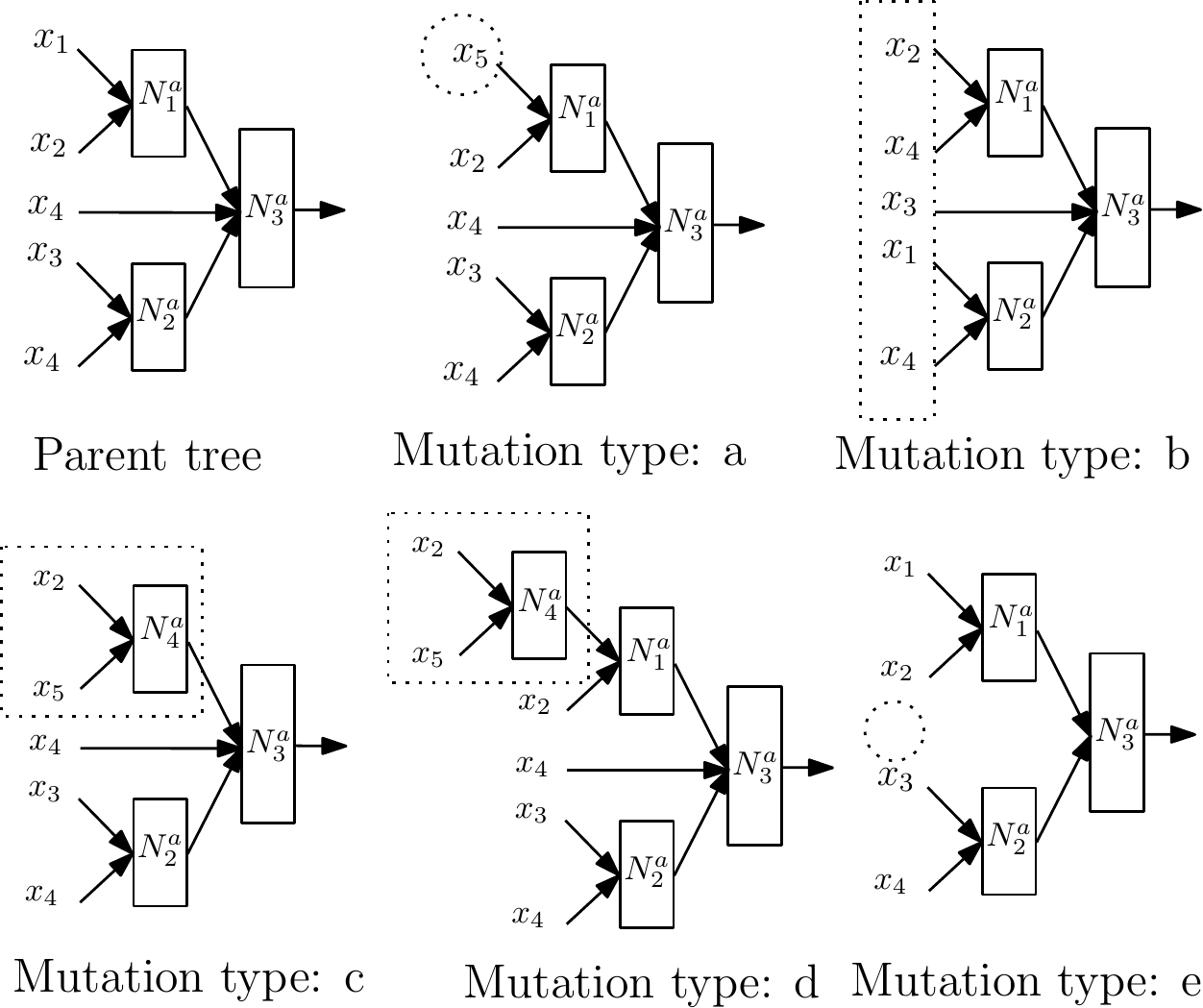}
		\label{fig_tree_mutation}
	}
	\caption{\color{black}MOGP Operations: (a) Crossover between two parent trees a and b.  (b) A total five types of mutation of the parent tree.}
\end{figure}

\subsection{Parameter Tuning}
\label{sec_hfit_DE}
In {\color{black}the} structure tuning phase, an optimum phenotype (HFIT) was derived with the parameters being initially fixed by random {\color{black}guesswork}. Hence, the obtained phenotype was further tuned in the parameter tuning phase by using a parameter optimization algorithm. To tune the parameters of the derived phenotype, its parameters were mapped onto a genotype, i.e., onto a real vector, called {\color{black}a} solution vector. The selection of the best phenotype in a single objective training was solely based on {\color{black}a} comparison of the RMSEs. However, selecting a solution in a multiobjective training is a difficult choice. In this work, after the multiobjective training of HFIT, the best solution for parameter tuning was picked from the Pareto front. Strictly, the solution that gave the best RMSE among the solutions {\color{black}marked rank-one} in the Pareto-optimal set was chosen. Fig.~\ref{fig_pareto_front} is an illustration of the solutions that belong to {\color{black}the} Pareto-front{\color{black}.} 

The genotype mapping of a T1FIS and a T2FIS differ only {in regard to} their number of parameters. {\color{black} In HFIT,} a T1FIS uses the MF mentioned in~\eqref{eq_type1_MF}, which has two arguments $ m $ and $ \sigma $ and each rule in T1FIS has $ {d^i} + 1$ variables {\color{black}in} the consequent part as referred to in~\eqref{eq_type1_consequent}, where $ d^i $ is the number of inputs to the $ i $-th rule. On the other hand, {\color{black} a T2FIS uses IT2FSs, which} are bounded by LMFs and UMFs (Fig.~\ref{fig_MF_T2}) and have two Gaussian means $ m_1 $ and $ m_2 $ and a variance $ \sigma $ {\color{black}to} be optimized. The Gaussian means $ m_1 $ and $ m_2 $ for type-2 Gaussian MF~\eqref{eq_GaussT2MF} were defined as:
{\color{black}\begin{equation*}
m_1 = m + \lambda \sigma\hspace*{1em} \text{ and }\hspace*{1em} m_2 = m - \lambda \sigma, 
\end{equation*}  }   
where $ \lambda \in [0,1]$ is a random variable taken from uniform distribution and $ m $ is the center of {\color{black}the} Gaussian means $ m_1 $ and $ m_2 $ taken from [0, 1]. Similarly, the variance $ \sigma $ of type-2 Gaussian MF~\eqref{eq_GaussT2MF} was taken from [0, 1]. The consequent part of {\color{black}the} T2FIS was computed  {\color{black}according to}~\eqref{eq_type2_consequent}, which led to $ 2 (d^i + 1) $ variables. 

\begin{figure}
	\centering
	\includegraphics[width=0.6\columnwidth]{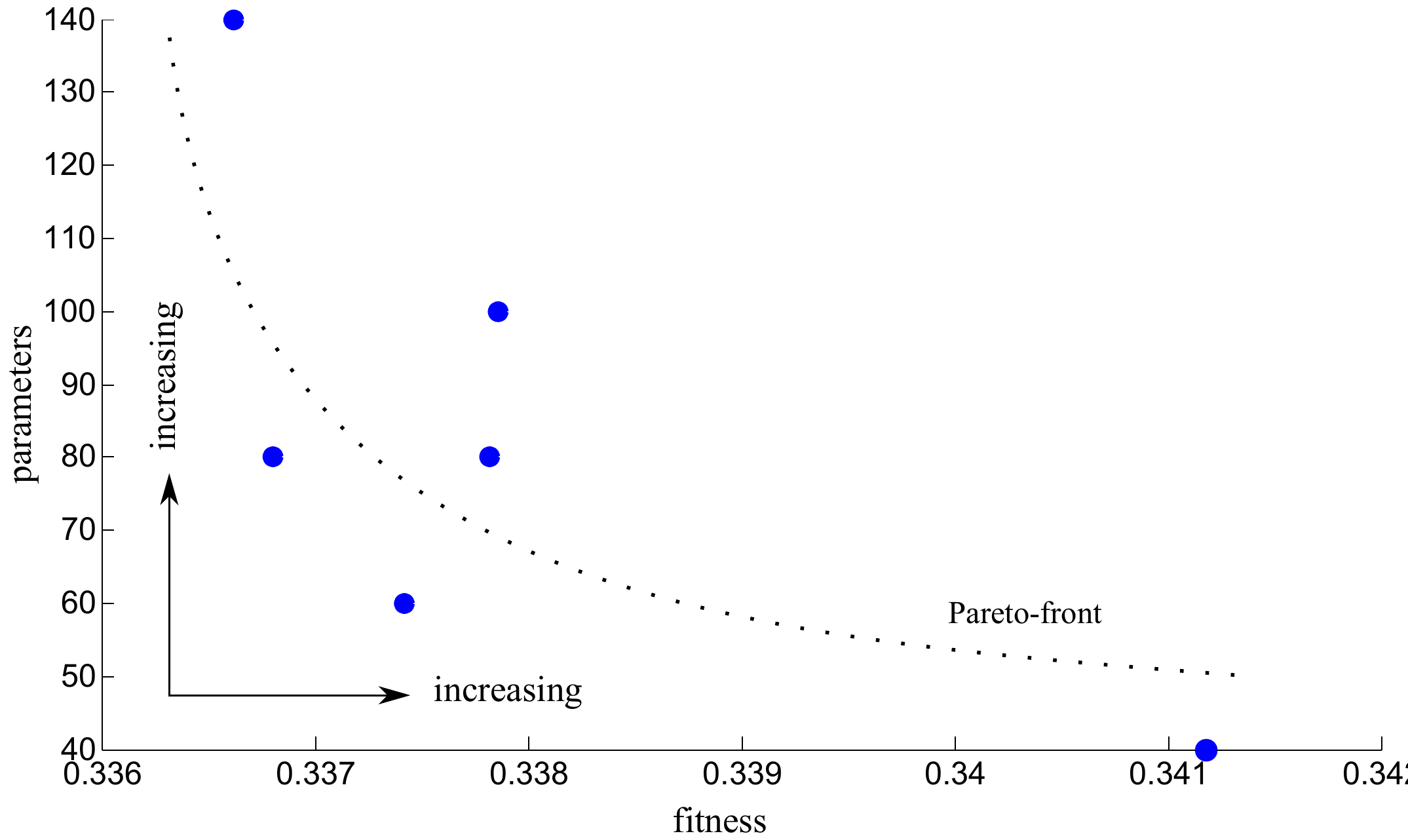}%
	\caption{FIS fitness versus FIS {\color{black}parameter} mapping across the Pareto-front. This graph was {\color{black}created} during a multiobjective training of the example--2 mentioned in Section~\ref{sec_fis_exp_2}.}
	\label{fig_pareto_front}
\end{figure}

Assume that an HFIT (a tree like Fig.~\ref{fig_HFIS_tree}) has $ k $ many nodes and each node in the phenotype takes $ 2 \le d^i \le tn$ inputs, where each input is partitioned into two fuzzy sets (MFs). {\color{black}Then, the number of the fuzzy sets at a node is $ 2d^i$. Since the number of inputs at a node is $ d^i $ and each input is partitioned into two fuzzy sets,  the number of rules at a node is $ 2^{d^i} $. Hence, the number of parameters at a T1FIS node is {\color{black}$[2(2d^i) + 2^{d^i}({d^i} +1)] $} and the number of parameters at a T2FIS node is {\color{black}$[3(2d^i) + 2^{d^i}(2({d^i} +1))] $}. Therefore, the total number of parameters in an HFIT is the summation of the number of parameters at all $ k $ nodes in the tree. For example, the number of parameters in the type-1 HFIT and type-2 HFIT shown in Fig.~\ref{fig_HFIS_tree} are 84 and 154 respectively.}

Assuming $ n $ is the total number of parameters in a tree, the genotype or the solution vector $ \text{\textbf{w}}  $ corresponding to the tree {\color{black}(phenotype)} is expressed as: 
\begin{equation}
\label{eq_realVector}
\text{\textbf{w}} = \langle w_1,w_2,\ldots,w_n\rangle
\end{equation} 
{\color{black}Now}, to optimize parameter vector $ \text{\textbf{w}} $, {\color{black}a parameter optimizer} can be used: genetic algorithms~\cite{eiben2015ec}, evolution strategy~\cite{eiben2015ec}, artificial bee colony~\cite{karaboga2007powerful}, PSO~\cite{poli2007particle}, DE~\cite{das2016recent}, gradient-based algorithms~\cite{snyman2005practical}, backpropagation~\cite{werbos1990backpropagation}, Klaman-filter~\cite{haykin2001kalman}, etc. 	

In this work, {\color{black}the} differential evolution (DE) version ``DE/rand-to-best/1/bin''~\cite{qin2009differential} was used, which is a metaheuristic algorithm that uses a crossover operator inspired by the dynamics of {\color{black}``natural selection.''} The basic principle of the DE is as follows{\color{black}:} First, an initial population matrix $\text{W}^t =  (\text{\textbf{w}}_1,\text{\textbf{w}}_2,\ldots,\text{\textbf{w}}_P)$ at the iteration $  t = 0 $ is randomly initialized. The population $\text{W}^t $ contains $ P $ many solution vectors. A solution vector $\text{\textbf{w}} $ in the population is an $n$-dimensional vector representing the free parameters of an HFIT. Secondly, the population $\text{W}^{t+1}$ is created using binomial trials. Hence, to create a new solution vector for the population $\text{W}^{t+1}$, three distinct solution vectors $\text{\textbf{w}}^a,$ $\text{\textbf{w}}^b$, and $\text{\textbf{w}}^c$ and the best solution vector $ \text{\textbf{w}}^g $ are selected from the population $\text{W}^{t}$. Then, for a random index $k \in [1, n]$ and for the selected trial vector $\text{\textbf{w}}^a= \langle w_1^a,w_2^a,\ldots,w_n^a\rangle$, the $ j $-th variable of the modified trial vector $\text{\textbf{w}}^{a'}$ {\color{black}are} created as:  
\begin{equation}
\label{eq_deEquation}
w_j^{a'}  =
\begin{dcases}
\begin{split}
w^a_j + F(w^g_j - w^a_j ) + F (w^b_j - w^c_j ),\hspace*{2em} 
\end{split}
& r_j < cr \parallel j = k  \\
\begin{split}
w_j^a,
\end{split}
& r_j \ge cr 
\end{dcases}
\end{equation}
where $ r_j \in [0, 1] $ is a uniform random sample, $cr \in [0, 1]$ is the crossover rate, and $ F \in [0, 2]$ is the differential weight.
Similarly, all the variables $ j = 1 \text{ to }  n$ of the trial vector $\text{\textbf{w}}^{a}$ is created using~\eqref{eq_deEquation}. After creation of the modified trial vector $\text{\textbf{w}}^{a'}$, it is \textit{recombined} as:
\begin{equation}
\label{eq:deRecombination}
\text{\textbf{w}}^a = \left\lbrace 
\begin{array}{ll}
\text{\textbf{w}}^{a'},\hspace*{2em} & E(\text{\textbf{w}}^{a'}) < E(\text{\textbf{w}}^a) \\
\text{\textbf{w}}^a,\hspace*{2em} & E(\text{\textbf{w}}^{a'}) \ge E(\text{\textbf{w}}^a)\\
\end{array}		
\right.
\end{equation}
where {\color{black}$ E(.) $} is the function that returns the fitness of a solution vector using~\eqref{eq_rmse}.
{\color{black} {\color{black}In DE} {\color{black},} operators{\color{black}, such as} \emph{selection}, \emph{crossover}, and \emph{recombination} were repeated until a {\color{black}satisfactory} solution vector $ \mathrm{\mathbf{w}}^* $ was found or no improvement was observed compared to an obtained solution over a fixed period (100 DE iterations).}

\section{Theoretical Evaluation}
\label{sec_evaluation_the}
Efficiency of the proposed HFIT comes from a combined influence of three basic operations involved in the model's development: tree construction through MOGP, combining several low-dimensional fuzzy systems in a hierarchical manner, and parameters tuning through differential evolution (DE). Hence, HFIT bears many distinguished properties that define its prediction efficiency compared to many models invoked from literature for comparison. Following are the HFIT's properties: 1) Convergence ability of the evolutionary class algorithms (EA) or for that matter MOGP. 2) Approximation ability of the evolved hierarchical fuzzy system (tree model). 3) Convergence ability of DE in tree's parameters tuning. Subsequent discussions theoretically analyze each of these properties one-by-one.
\subsection{Optimal tree structure through MOGP convergence}
\label{sec:MOGP_convergence_proof}
Evaluating the convergence of evolutionary class algorithms has been a challenging task because of their stochastic nature. Theoretical studies of EAs performed through various perspectives show that indeed an optimal solution is possible in a finite time. Initially, Goldberg and Sergret~\cite{goldberg1987finite} showed convergence property of GA using a finite Markov chain analysis, where they considered GA with a finite population and recombination and mutation operators. 

A different viewpoint of MOGP convergence (EAs in general) can be referred to as by using Banach fixpoint theorem described in~\cite{szalas1993contractive}. Banach fixpoint theorem~\cite{Banach1922} states that on a metric space a constructive mapping $ f $ has a unique fixpoint, i.e., for an element $ x $, $ f(x) = x $. Therefore, Banach fixpoint theorem can explain MOGP convergence with only assumption that there should be an improvement of the population (not necessarily of the optimal solution) from one generation to another. Banach fixpoint theorem also indicates that if MOGP semantics is to be considered as a transformation between one population to another and if it is possible to obtain a metric space in which transformation is constructive, then MOGP converges to a optimal population $ \text{W}^* $, i.e., to a population containing optimal solution. 

A mapping $ f $ defined on elements of ordered pair set $ S $ is constructive if the distance between $ f(x) $ and $ f(y) $ is less than $ x $ and $ y $ for any $ x,y \in S $. Now, distance mapping $ \delta:S \times S \rightarrow \mathbb{R} $ is a metric space iff for any $ x,y \in S $ the following condition satisfy:
\begin{itemize}
	\item $ \delta(x,y) \ge 0$ and $\delta(x,y) = 0 $ if $ x = y $
	\item $ \delta(x,y) =  \delta(y,x) $
	\item $ \delta(x,y) +  \delta(y,z) \ge  \delta(x,z) $
\end{itemize}
Let $ \langle S,\delta \rangle $ be a metric space and  $ f:S \rightarrow S$ be a mapping, then $ f $ is \textit{constructive} iff there is a constant $ \epsilon \in [0,1) $ such that for all $ x,y \in S $
\begin{equation}
\delta(f(x),f(y)) \le \epsilon  \delta(x,y)
\end{equation}
Therefore, for Banach theorem's formulation, the completeness of the metric space needs to be defined. Now, metric space elements $ p_0, p_1, \ldots $ are a Cauchy sequence iff for any $ \epsilon > 0 $, there exist $ k $ such that for all $ m,n > k,$ $\delta(p_m, p_n) < \epsilon $. It also follows that, if such Cauchy sequence $ p_0, p_1, \ldots $ has a limit $ p = \lim_{n \rightarrow \infty} p_n $, then metric space is complete.
\begin{mytherm}
	For a complete metric space $ \langle S,\delta \rangle $ and constructive mapping $ f:S \rightarrow S$, mapping $ f  $ has a unique fixpoint $ x \in S $ such that for any $ x_0 \in S $
	$$ x = \lim_{i \rightarrow \infty} f^i(x_0)$$
	where $ f^0(x_0) = x_0 $ and $ f^{i+1}(x_0) = f(f^i(x_0)) $ 
\end{mytherm}
\begin{proof}
	A proof of Banach theorem can be found in~\cite[p. 60]{dixmier1984general} described as method of successive approximation.
\end{proof}   

In this article, it is necessary to show that if a metric space $ S $ for MOGP population can be obtained, then any constructive mapping $ f $ in MOGP will contain a unique fixpoint. The proposed MOGP has a fixed population size (say $ n $), i.e., each population contain $ n $ individuals, and in each generation, the total fitness of the population is expected to increase. Let $ \theta $ be a function that computes the fitness of a population, which is expressed as:
\begin{equation}
\theta(\text{W}) = \frac{1}{n} \sum_{\text{\textbf{w}}_i \in \text{W}} \frac{\sum_{\text{\textbf{w}}_i \in \text{W}} E(\text{\textbf{w}}_i)}{E(\text{\textbf{w}}_i)} 
\end{equation}   
where function $ E(\text{\textbf{w}}_i) $ evaluates RMSE of each $ \text{\textbf{w}}_i $. Now, distance mapping $ \delta:S \times S \rightarrow \mathbb{R} $, where $ S $ is a set of MOGP populations, can be defined as:
\begin{equation}
\delta(\text{W}_1, \text{W}_2) = \left\lbrace
\begin{array}{ll}
0, &  \text{W}_1 = \text{W}_2\\
|\theta (\text{W}_1)| + |\theta (\text{W}_2)|, & \text{W}_1 \ne \text{W}_2
\end{array}
\right.
\end{equation}
It follows that
\begin{itemize}
	\item $ \delta(\text{W}_1,\text{W}_2) \ge 0$ and $\delta(\text{W}_1,\text{W}_2) = 0 $ if $ \text{W}_1 = \text{W}_2 $ holds for any population $ \text{W}_1 $ and $ \text{W}_2 $ in MOGP.
	\item $ \delta(\text{W}_1,\text{W}_2) =  \delta(\text{W}_2,\text{W}_1) $ is obvious and 
	\item $ \delta(\text{W}_1,\text{W}_2) + \delta(\text{W}_2,\text{W}_3) = |\theta (\text{W}_1)| + |\theta (\text{W}_2)| + |\theta (\text{W}_2)| + |\theta (\text{W}_3)| \ge |\theta (\text{W}_1)| + |\theta (\text{W}_3)| = \delta(\text{W}_1,\text{W}_3) $
\end{itemize}
Therefore, MOGP has a metric space $ \langle S, \delta \rangle $. Now, it only remains to show that the MOGP follows a constructive mapping $ f: S \rightarrow S$, i.e., in each subsequent generation of MOGP, an improvement is possible. Altenberg~\cite{altenberg1994evolution} showed that by maintaining genetic operators, such as selection, crossover, and  mutation, the evolvability of genetic programming can be increased. Additionally, Altenberg~\cite{altenberg1994evolution} analyzed the probability of a population containing fitter individuals than the previous population and offered the subsequent proof. It was observed that even for a random crossover operation, genetic programming evolvability can be ensured. It is then necessary to say that, indeed an MOGP can produce fitter population than the previous ones. 

Let's depart from MOGP operations descriptions to continue with Banach theorem since it is now known that MOGP offers constructive mapping $ f:S\rightarrow S $, for which $ t $-th iteration population offers constructive mapping. In other words, $ \theta(\text{W}^t) < \theta(\text{W}^{t+1})$, i.e., mapping $ f(\text{W}^t) = \text{W}^{t+1} $ holds. It follows that 
$$ \delta(f(\text{W}^t_1),f(\text{W}^t_2))  < \delta(\text{W}^t_1,\text{W}^t_2)  $$ 
Moreover, it satisfies Banach fixpoint theorem. Hence, 
\begin{equation}
\text{W}^* = \lim_{i\rightarrow \infty} f^i(\text{W}^0)
\end{equation}
It indicates that MOGP converges to a population $ \text{W}^* $, which is a unique fixpoint in a population space. 
\begin{remark}
	It is evident from MOGP operation that it produces an optimal tree structure from a population space. Although obtaining optimality in the tree design using MOGP is sufficient to claim the formation of a function that can approximate to a high degree of accuracy, it is necessary to investigate the approximation capability of the hierarchical fuzzy system developed in the form of a tree structure. 
\end{remark}    
\subsection{Approximation ability of hierarchical fuzzy inference tree}
\label{sec:approximation_HFIT}
This Section describes the approximation capability of an HFIT, which is a result of MOGP operation. Theoretical studies of special cases of the hierarchical fuzzy systems are provided in~\cite{wang1999analysis,joo2005class}. Whereas, the proposed HFIT produces a general hierarchical fuzzy system. In HFIT, not only a cascaded hierarchy of fuzzy system (a fuzzy subsystem takes input only from its previous fuzzy subsystem~\cite{wang1999analysis}) can be produced, but a general hierarchical fuzzy system, in which a fuzzy subsystem can take inputs from any previous layer fuzzy subsystem, can be produced. A hierarchical fuzzy system described in~\cite{zeng2005approximation} resembles the hierarchical fuzzy system produced by HFIT. To show the approximation capability of the proposed HFIT, it requires coming to the conclusion that the proposed HFIT is analogous to the hierarchical fuzzy system described by Zeng and Keane~\cite{zeng2005approximation}. 

Let's perform an analogy between the proposed HFIT and the concept of a natural hierarchical fuzzy system described by Zeng and Keane~\cite{zeng2005approximation}. To show such an analogy, at first, it needs to establish the definition of the natural hierarchical structure of a continuous function, then it will be necessary to show that, for any such continuous function, a hierarchical fuzzy system exists. 

Let's take the example of the HFIT shown in Fig.~\ref{fig_HFIS_tree}, which can be represented as natural hierarchical structure of a continuous function. The tree in Fig.~\ref{fig_HFIS_tree} gives the output $ y_3 $ from node $ N_3 $. Moreover, the tree in Fig.~\ref{fig_HFIS_tree} gives the following functions:
$$ y_3 = N_3(y_1,y_2,x_3) \quad  y_1 = N_1(x_1,x_2)  \quad y_2 = N_2(x_4,x_5)$$ 
It can also be expressed as:
\begin{equation}
\label{eq:tree_epression}
N(x_1,x_2,x_3,x_4,x_5) = N_3[N_1(x_1,x_2),N_2(x_4,x_5),x_3]
\end{equation}
It follows that, for a given function $ y = N(x_1,x_2,x_3,x_4,x_5) $, if there exist functions $ N_3,N_2,N_1 $ such that function~\eqref{eq:tree_epression} can be obtained, then function $ N(x_1,x_2,x_3,x_4,x_5) $ can be represented as hierarchical structure. 

For simplicity's sake, let's take the case of a two-stage tree, where the top layer node is denoted by $ N^1 $ and its output is by $ y $. Similarly, second layer nodes are denoted by $ N^2_i $ and their outputs are by $ y^2_i $, for $ 1\ge i \le m$. Therefore, a natural hierarchical structure can be defined as:   
\begin{mydef}[Natural Hierarchical Structure]
	Let $ y = N(x_1,\ldots,x_n) $ be a multi-input-single-output continuous function with $ n $ input variables $ \textbf{\text{x}} = \langle x_1,\ldots, x_n \rangle $ defined on input space $ U = \prod_{i=1}^n U_i \subset \mathbb{R}^n $ and the output $ y $ defined on the output space $ V \subset \mathbb{R} $. If there exist $ m + 1 $ continuous functions
	\begin{equation}
	\begin{aligned} 
	y  & = N^1(y^2_1,\ldots,y^2_m, x^1_i, \ldots, x^1_{d_1^i})  \\
	y^2_j &= N^2_j(x^{2j}_i, \ldots, x^{2j}_{d_{2j}^i}) 
	\end{aligned}
	\end{equation} 
	and the functions have inputs $ x^1_{d_1^i} $ and $ x^2_{d_{2j}^i} $, where $ d_1 < n $ and $ d_{2j} < n $ are input dimensions at the top and second stage of hierarchy, respectively, such that 
	\begin{equation}
	N(x_1,\ldots,x_n)  = N^1[N^2_1(x^{21}_i, \ldots, x^{21}_{d_{21}^i}),\ldots, N^2_m(x^{2m}_i, \ldots, x^{2m}_{d_{2m}^i}) , x^1_i, \ldots, x^1_{d_1^i}]  
	\end{equation} 
	then $ N(x_1,\ldots,x_n) $ is a continuous function with natural hierarchal structure.	 
\end{mydef}
Such form of natural hierarchical structure also possesses separable or arbitrarily separable hierarchical structural property, i.e., the individual functions can be decomposed~\cite{zeng2005approximation}. Now, from Kolmogorov's Theorem~\cite{kolmogorov1963representation}, the following can be stated: Any continuous function $ N(x_1,\ldots,x_n) $ on $ U = \prod_{i=1}^n [\alpha_1,\beta_i]$ ($ \alpha_i$ and $\beta_i $ define the input range) can be represented as a sum of $ 2n + 1 $ continuous functions with an arbitrarily separable hierarchical structure. This statement concludes to the following theorem.
\begin{mytherm}
	\label{theorem:fuzzy_universsal}
	Let $ N(\text{\textbf{x}}) $ be any continuous function on $ U = \prod_{i=1}^n [\alpha_i,\beta_i]$ and its hierarchical structure representation be $ N(\text{\textbf{x}}) = N^1[N^2(\text{\textbf{x}}),\ldots,N^m(\text{\textbf{x}})] $, in which $ N^j(\text{\textbf{x}})(j = 1, \ldots, m) $ are continuous functions with natural hierarchical structure, then for any given $ \epsilon > 0 $, there exists a hierarchical fuzzy system
	$$ G(\text{\textbf{x}}) = G^1[G^2(\text{\textbf{x}}),\ldots,G^m(\text{\textbf{x}})] $$  
	which has the same natural hierarchical structure as $ N(\text{\textbf{x}}) $ such that
	\begin{equation}
	\|N -G \|_\infty < \epsilon
	\end{equation}
	\begin{equation}
	\|N^i - G^i \|_\infty < \epsilon  \quad i = 0,1,\ldots,m
	\end{equation}
	and the same holds between the sub-functions of $ N^i(\text{\textbf{x}}) $ and the fuzzy subsystems of $ G^i(\text{\textbf{x}}) (i=0,1,\ldots,m) $.
\end{mytherm}
\begin{proof}
	Proof of Theorem~\ref{theorem:fuzzy_universsal} can be found in~\cite{zeng2005approximation}.
\end{proof}	
\begin{remark}
	It is to note that Theorem~\ref{theorem:fuzzy_universsal} shows that hierarchical structure of fuzzy systems is universal approximators. Therefore, they can approximate any continuous function $ N(\text{\textbf{x}}) $ to any degree of accuracy as-well-as they can approximate each component of that function. Hence, the proposed HFIT that can form a natural hierarchical structure can achieve universal structure approximation.   
\end{remark}
Another property of the proposed HFIT is the parameter tuning, which is performed by a global optimizer (e.g., DE was applied in this research). Hence, it is required to investigate the convergence ability of the DE algorithm in parameter tuning of HFIT.

\subsection{Optimal parameter through differential evolution convergence}
Convergence property and efficiency of DE is well studied~\cite{xue2005modeling,zhang2009theoretical}. A probabilistic viewpoint of DE convergence followed by a description of global convergence condition for DE is described in~\cite{hu2013sufficient}. They show that indeed DE converges to an optimal solution. Similarly, Zhang and Sanderson~\cite{zhang2009theoretical} studied the various property of DE, such as mutation, crossover and recombination operators that influence the DE convergence. DE follows a similar property as of EA class algorithms described in Section~\ref{sec:MOGP_convergence_proof}. Hence, its global convergence ability is not different than the one described for MOGP, and indeed it finds an optimal parameter vector for HFIT.

\subsection{Comparative study of HFIT with other models}
\label{sec_the_comparison}
The proposed HFIT learns knowledge contained in the training data through adaptation in its structure and the rules generated at its nodes. Such a process of learning/acquiring knowledge from data is somehow similar to the models having network-like layered architecture, i.e., ANFIS-like approaches, which usually have 4 or 5 or 6 layered network structure. However, HFIT's strength comes from its adaptive structure formation, whereas most of the network-like models have fixed layer structure. 

\subsubsection{Flexible structure formation}
Specifically, the models depending on layered structure (e.g., HyFIS, DENFIS, D-FNN, EFuNN, FALCON, GNN, SaFIN, SONFIN, SuPFuNIS, eT2FIS, IT2FNN-SVR-N/F, McIT2FIS-UM/US, RIT2FNS-WB, SEIT2FNN, SIT2FNN, TSCIT2FNN, etc.) can only provide adaptation in the number of generated rules in their hidden layer by keeping the input (first) and output (last) layer fixed. Network-like model's fixed layered architecture in some sense limits their representational flexibility as compared to HFIT. 

In Section~\ref{sec:approximation_HFIT}, it was shown that HFIT has the capability of representing any continuous function in any natural and arbitrarily separable hierarchical form. Therefore, it can be said that the network-like models that grow rules only in one direction have a shortfall in structural representation compared to HFIT, which can grow in layer-wise as-well-as breadth-wise. 

\subsubsection{Diverse fuzzy rules formation}
Additionally, the interaction of one RB to another through the structural representation is what sets HFIT apart from the other models, which generate only a single RB and do not have the interaction as it is in HFIT. Moreover, nodes in HFIT take difference input's combination govern by MOGP. Therefore, HFIT nodes exhibit heterogeneity, which drives the formation diverse rules in the nodes of HFIT. Whereas, rules in network-like models use same combination inputs while adding rules in the hidden layer during their training process.   

\subsubsection{Automatic fuzzy set selection}
Adaptation in the most of the network-like models is due to the input space partitioning (usually for choosing the number of membership function at the second layer) in two or three fixed fuzzy sets or by using some clustering method, which directly influences the number of rules to form in the third layer (usually called rule layer). The necessity of predefining the number of clusters is basic disadvantage with the clustering based partitioning. Some of the practices in the clustering based partitioning, like the one in SaFIS, are devoted to improving the clustering algorithms to avoid the requirements of such predefinition. However, the overlapping of the membership function of the fuzzy sets is another common problem with clustering based input space partitioning~\cite{tung2011safin}. In~\cite{paul2002subsethood}, authors pointed out four different cases of membership function's overlapping and proposed subsethood method to transmit the overlapping information to the rules layer.   

On the other hand, HFIT does not use clustering to determine input space partitions. Instead, each input at each HFIT's node is partitioned into two fuzzy sets, which is eventually determined by MOGP through evolution. Section~\ref{sec:MOGP_convergence_proof} shows that MOGP finds an optimum solution through its iterations and the use of genetic operators. Hence, MOGP at it best avoids the overlapping of the membership function of the fuzzy sets and also eliminates the requirement of an external agent for input space partitioning.  

\subsubsection{Minimal feature set selection}
Feature selection is another important aspect of HFIT. The network-like models such as EFuNN and DENFIS also does feature selection externally (say by external agents). However, in a sense, feature selection in such kind of models do not have direct participation in the structural representation of knowledge contained in training data. Whereas, feature selection is an integral part HFIT's learning process. Hence, feature selection performed by HFIT incorporates knowledge contained in training data into its structural representation in an explicit way compared to other network-like models. Since an external agent performs feature selection in the network-like models and many other models do not even perform feature selection, they are disadvantageous compared to HFIT when in comes to solving high-dimensional problems.

\subsubsection{Parameter tuning}
Finally, most of the models such as HyFIS, DENFIS, SaFIN, SONFIN, SEIT2FNN, McIT2FIS-UM/US, etc. employ gradient-based methods (e.g., backpropagation) for the parameter tuning. The gradient-based techniques are known as local optimizers, which lacks exploration capability compared to global optimizers (e.g., DE)~\cite{deStorn1997}. HFIT employs DE for its parameter optimization. When it comes to comparing models theoretically, it is not necessary to go deep in parameter tuning debate since such parameter tuning method like DE can also be applied to other models and backpropagation can be applied to HFIT. However, at present scenario, a combined effort of the proposed HFIT model, in this article, have an advantage compared to other models.

%

{\color{black}
\section{Empirical {\color{black}Evaluation}}
\label{sec_evaluation_emp}}
This section describes the evaluated results of the proposed algorithms T1HFIT$^{\text{S}}$, T1HFIT$^{\text{M}}$, T2HFIT$^{\text{S}}$, and T2HFIT$^{\text{M}}$ on six example problems. Assume that the datasets in the examples are of the form: $ (\text{\textbf{X}},\text{\textbf{d}}) $, where $ \text{\textbf{X}} = ( \text{\textbf{x}}_1,\text{\textbf{x}}_2,\ldots,\text{\textbf{x}}_N ) $ is the set of the input vectors and  $\text{\textbf{d}} = \langle d_1,d_2,\ldots,d_N \rangle $ is the desired output vector. Here, the dataset has $ N $ input--output patterns (pairs) and if the vector $\text{\textbf{y}} = \langle y_1,y_2,\ldots, y_N \rangle $ is the predicted output vector{\color{black}, then} the performance of an algorithm for the dataset $ ( \text{\textbf{X}},\text{\textbf{d}} ) $ can be measured using RMSE $ E $ as defined in~\eqref{eq_rmse} and correlation coefficient $ r $ between the desired output vector $ \text{\textbf{d}} $ and $ \text{\textbf{y}} $ {\color{black}as:}
\begin{equation}
\label{eq-corr}
r = \frac{\sum_{i = 1}^{N}\left(d_i - \bar{d} \right) \left(y_i - \bar{y} \right) }{ \sqrt{\sum_{i = 1}^{N}\left(d_i - \bar{d} \right)^2 \sum_{i = 1}^{N}\left(y_i - \bar{y} \right)^2}}
\end{equation} 
where  $ \bar{d} $ and $ \bar{y} $ are the means of the vectors $ \text{\textbf{d}} $ and $ \text{\textbf{y}} $. For {\color{black}simplicity's sake}, the training and the test RMSEs were represented as $ E_n $ and $ E_t $, respectively. Similarly, the training and the test correlation coefficients were represented as $ r_n $ and $ r_t $, respectively. Additionally, the model's complexity $ c(\text{\textbf{w}}) $ {\color{black}and training time (in minutes) were reported. The reported training time included the time taken to create a tree structure, tune the tree parameters, partition the dataset (file input--output operations), write the developed model to a file, display {\color{black}the} tree on {\color{black}a} GUI, and compute RMSE and correlation coefficient.} The parameter setting mentioned in Table~\ref{tab_parameter_fis} was used {\color{black}to train} the proposed algorithms, which was developed as {\color{black}a} software tool and is available at http://dap.vsb.cz/sw/hfit/. The experiments were conducted on {\color{black}a} Windows Server R2 that had {\color{black}a} 20 core and 700 GB RAM. Each run of experiments was conducted with the random seeds generated from the system. The proposed algorithms were compared  with the algorithms collected from the literature (Table~\ref{tab_lit_fis}).

\begin{table}[!ht]
	\centering
 	{\footnotesize 	
	\caption{Parameter Setting for the Experiments}
	\label{tab_parameter_fis}
	\begin{tabular}{ll}
    	\toprule
		Algorithm training parameter & Value  \\
		\hline
		Maximum depth (layers) of a tree  &  4 \\			
		Maximum inputs to an FIS node  & 4 \\			
		Membership function search range & [0,1] \\		
		GP population & 50\\			
		CP mutation probability $ pm $ & 0.2  \\		
		GP crossover probability $ pc = 1 - pm$ & 0.8 \\		
		GP mating pool size & 25 \\			
		GP tournaments selection size &  2 \\	
		GP iterations & 500	\\
		DE population & 50 \\
		DE mutation factor $ F $ & 0.7 \\ 
		DE crossover factor $ cr $ & 0.9 \\			
		DE iterations & 5000  \\
		\bottomrule
	\end{tabular}}
\end{table}	

\begin{table*}
	\centering
	{\scriptsize 	
	\caption{Descriptions of the Existing FIS Algorithms Adopted for the Performance Comparisons}
	\label{tab_lit_fis}
	\begin{tabular}{ ll  l  p{7cm}  l p{3.5cm}  }
		\toprule 
		FIS & Algorithm &  Ref. & Description &  Type  & Parameter tuning\\
		\hline \multirow{13}{*}{\begin{sideways}{Type--1}\end{sideways}}
		& DENFIS & \cite{kasabov2002denfis} & Dynamic evolving neural-fuzzy inference system & TSK & Least-square estimator \\
		& D-FNN & \cite{wu2000dynamic} & Dynamic fuzzy neural networks & TSK & Backpropagation algorithm \\
		& EFuNN & \cite{kasabov2001evolving} & Evolving fuzzy neural networks & Mamdani & Widrow--Hoff least square \\
		& FALCON & \cite{lin1997art} & ART-based fuzzy adaptive learning control network & $--$ & Backpropagation algorithm \\
		& GNN & \cite{zhang2008granular} & Granular neural networks & $--$ & Genetic algorithm  \\
		& H-TS-FS & \cite{chen2007automatic} & Hierarchical Tukagi--Sugno fuzzy system & TSK & Evolutionary programming \\
		& HyFIS & \cite{kim1999hyfis} & Hybrid neural fuzzy inference system & $--$ & Gradient descent learning \\
		& IFRS and AFRS & \cite{duan2002multilevel} & Incremental and aggregated fuzzy relational systems & Mamdani & Backpropagation algorithm \\
		& RBF-AFA & \cite{cho1996radial} & Radial basis function based adaptive fuzzy systems & TSK & Gradient descent learning \\
		& SaFIN & \cite{tung2011safin} & Self-adaptive fuzzy inference network & Mamdani & Levenberg-Marquardt method\\
		& SONFIN & \cite{juang1998online} & Self-constructing neural fuzzy inference network & TSK & Backpropagation algorithm \\
		& SuPFuNIS & \cite{paul2002subsethood} & Subsethood-product fuzzy neural inference system & $--$ & Gradient descent learning \\
		& SVR-FM & \cite{chiang2004support} & Support-vector regression fuzzy model & TSK & Support vector regression \\
		\hline \multirow{11}{*}{\begin{sideways}{Type--2}\end{sideways}}
		& eT2FIS & \cite{tung2013et2fis} & Evolving type-2 neural fuzzy inference system & Mamdani & Gradient descent learning \\
		& IT2FNN-SVR-N/F & \cite{juang2010interval} & IT2fuzzy-NN-support-vector regression-fuzzy and numeric & TSK & Support vector regression \\
		& McIT2FIS-UM/US & \cite{das2015evolving} & Metacognitive interval type-2 neuro-fuzzy inference system & TSK & Gradient descent learning \\
		& NNT1FW and NNT2FW & \cite{angelov2004approach} & Type-1 and type-2 fuzzy  backpropagation neural networks & TSK & Backpropagation algorithm \\
		& RIT2FNS-WB & \cite{juang2013reduced} & Reduced IT2NFS-weighted bound-set & TSK & Gradient descent learning \\
		& MRIT2NFS & \cite{juang2013reduced} & Reduced IT2NFS-weighted bound-set & Mamdani & Gradient descent learning \\
		& SEIT2FNN & \cite{juang2008self} & Self-evolving IT2FIS & TSK & Kalman filtering algorithm \\
		& SIT2FNN & \cite{lin2014simplified} & Simplified Interval Type-2 Fuzzy Neural Networks & TSK & gradient descent learning \\
		& T2FLS & \cite{mendel2002} & Interval type-2 fuzzy logic system (TSK and singleton) & TSK & $--$ \\
		& T2FLS-G & \cite{mendel2004computing} & Gradient-descent based IT2FIS tuning & TSK & Derivation-based learning \\
		& TSCIT2FNN & \cite{lin2014tsk} & Compensatory interval type-2 fuzzy neural network & TSK & Kalman filter algorithm \\
		\bottomrule
	\end{tabular}}
\end{table*}

\subsection{Example 1---System Identification}
Online identification of the nonlinear system is a widely studied problem. The significance of this problem is evident from its usage in the literature for the validation of the approximation algorithms~\cite{narendra1990identification,juang2008self,juang2010interval,lin2014tsk,das2015evolving}. The nonlinear system identification of the plant is described by the following nonlinear difference equation:
\begin{equation}
\label{eq_plant}
y_p(k+1) = \frac{y_p(k)}{1+y_p(k)^2} + u^3(k)
\end{equation}
where [$u(k), y_p(k)$] is the input--output pair of the single input and the single output plant at the time $k$ and $y_p(k+1)$ is the one step ahead prediction. Hence, the objective is to predict $y_p(k+1)$  of the system based on the sinusoidal input $ u(k) = \sin(2 \pi k/100) $ and the current output $y_p(k) $. Let us assign the input $ x_1  = u(k)$ and the input $ x_2 = y(k)$.

{\color{black}The training  patterns were generated with $ k =1,\ldots,200$ and $ y_p(1) = 0$. Similarly, the test patterns were generated for $ k =201,\ldots,400$ as mentioned in~\cite{juang2013reduced}. Therefore, for the training, the inputs were $u(k)$ and  $y_p(k)$, and the desired output was $ y_p(k+1) $.} 
{\color{black}The}  {\color{black}training and test were repeated ten times. Such repetitions were performed mainly for assessing an average performance of the proposed algorithms, which is} shown in Table~\ref{tab_onsi_stat}. Since the experiments were repeated ten times, ten different models were obtained for each algorithm. The results of the best models (regarding their RMSEs) were compared with the {\color{black}best} results available in the literature (Table~\ref{tab_onsi_comp}).

The performance statistics, shown in Table~\ref{tab_onsi_stat}, {\color{black}are} {\color{black}evidence} of the efficiency of the proposed algorithms. {\color{black}They show} that the mean correlation coefficients $ r_n $ and $ r_t $  of training and test sets are $ 1.00 $ and $ 1.00 $, respectively, which {\color{black}indicate} that the algorithm consistently performed with a high accuracy. Moreover, such consistency of high accuracy is evident from the obtained small standard deviations (STD) of the training and test RMSEs and correlation coefficients (Table~\ref{tab_onsi_stat}). 

Interestingly, the Pareto-based  multiobjective training offered less complex models (the mean parameter count $ c(\text{\textbf{w}}) $ of T1HFIT$^{\text{M}}$ was 34.4 compared to 57.2 of T1HFIT$^{\text{S}}$ and $ c(\text{\textbf{w}}) $ of T2HFIT$^{\text{M}}$ was 90.4 compared to 152.0 of T1HFIT$^{\text{S}}$) with high accuracies (Table~\ref{tab_onsi_stat}). Additionally, the {\color{black}training} time taken by T1HFIT$^{\text{M}}$ and T2HFIT$^{\text{M}}$ was much {\color{black}less} than {\color{black}by} T1HFIT$^{\text{S}}$ and T2HFIT$^{\text{S}}$. Hence, the Pareto-based multiobjective was advantageous to use, which provided the option of choosing the best solution from a Pareto-front. An example of {\color{black}a Pareto-front} is shown in Fig.~\ref{fig_pareto_front}.    
\begin{table}
	\centering
	\caption{Performance Evaluation on System Identification (Example-1)}
	\subtable[Performance Statistics (10 repetitions)]{
	\setlength{\tabcolsep}{3pt}
	\begin{tabular}{llrrrr}
		\toprule
		&  & T1HFIT$^{\text{S}}$ & T1HFIT$^{\text{M}}$ & T2HFIT$^{\text{S}}$ & T2HFIT$^{\text{M}}$ \\
		\hline
		$E_n$ & Best & 0.0043 & 0.0041 & 0.0033 & 0.0028 \\
		& Mean & 0.0181 & 0.0257 & 0.0123 & 0.0184 \\
		& STD & 0.0167 & 0.0164 & 0.0074 & 0.0105 \\
		{\color{black}$r_n$} & Best & 1.00 & 1.00 & 1.00 & 1.00 \\
		& Mean & 1.000 & 0.999 & 1.000 & 1.000 \\
		& STD & 0.0006 & 0.0007 & 0.0001 & 0.0002 \\
		$E_t$ & Best & 0.0020 & 0.0041 & 0.0034 & 0.0028 \\
		& Mean & 0.0169 & 0.0262 & 0.0125 & 0.0187 \\
		& STD & 0.0173 & 0.0171 & 0.0076 & 0.0109 \\
		{\color{black}$r_t$} & Best & 1.00 & 1.00 & 1.00 & 1.00 \\
		& Mean & 1.000 & 0.999 & 1.000 & 1.000 \\
		& STD & 0.0006 & 0.0007 & 0.0001 & 0.0002 \\
		$ c(\text{\textbf{w}}) $ & Best & 20 & 20 & 72 & 36 \\
		& Mean & 57.2 & 34.4 & 152 & 90.4 \\
		{\color{black}Time} & Best & 3.21 & 1.52 & 7.82 & 3.23 \\
		     & Mean & 6.27 & 2.91 & 8.91 & 5.14 \\
		\bottomrule
	\end{tabular} 
	\label{tab_onsi_stat}
	}\quad~\quad
	\subtable[Performance Comparison]{
    \setlength{\tabcolsep}{3pt}
	\begin{tabular}{llccr}
		\toprule
		& Algorithm & $E_n$ & $E_t$ & $ c(\text{\textbf{w}}) $ \\
		\hline \multirow{5}{*}{\begin{sideways} {Type--1}\end{sideways}}
		& FALCON & 0.0200 &  & 54 \\
		& SaFIN &  & 0.0120 &  \\
		& SONFIN & 0.0080 & 0.0085 & 36 \\
		& \textbf{T1HFIT$^{\text{S}}$} & 0.0043 & 0.0043 & 60\\
		& \textbf{T1HFIT$^{\text{M}}$} & 0.0041 & 0.0041 & 40\\
		\hline \multirow{13}{*}{\begin{sideways} {Type--2}\end{sideways}}
		& T2FLS (singleton) & 0.0306 & $-$ & 120 \\
		& FT2FNN & 0.0388 & $-$ & 36 \\
		& T2FLS (TSK) & 0.0217 & $-$ & 120 \\
		& TSCIT2FNN & 0.0080 & $-$ & 34 \\
		& T2TSKFNS & $-$ & 0.0324 & 24 \\
		& T2FNN & $-$ & 0.0281 & 36 \\
		& SIT2FNN &  $-$ & 0.0241 & 36\\		
		& RIT2NFS-WB & 0.0073 & 0.0151 & 24 \\
		& MRI2NFS & 0.0042 & 0.0051 & 36 \\
		& T2FLS-G & 0.0214 & 0.0379 & 36 \\
		& SEIT2FNN & 0.0022 & 0.0022 & 84 \\
		& \textbf{T2HFIT$^{\text{S}}$} & 0.0033 & 0.0034 & 118\\
		& \textbf{T2HFIT$^{\text{M}}$} & 0.0028 & 0.0028 & 72\\
		\bottomrule
	\end{tabular}
	\label{tab_onsi_comp}
	}
\end{table} 

For the performance comparisons, the {\color{black}SaFIN result} was collected from~\cite{tung2011safin}, and FALCON and SONFIN from~\cite{juang2008self}. The results of T2FLS (singleton) and T2FLS (TSK) were obtained from~\cite{juang2008self}; FT2FNN, TSCIT2FNN, T2TSKFNS, and T2FNN from~\cite{lin2014tsk}; SEIT2FNN, MRI2NFS, RIT2NFS-WB, and T2FLS-G from~\cite{juang2013reduced}; and SIT2FNN from~\cite{lin2014simplified}. Table~\ref{tab_lit_fis} contains {\color{black}a} detailed description of these algorithms.

Two parameters may be used for comparing the algorithms: 1) the training and test RMSEs and 2) the {\color{black}parameter count} $ c(\text{\textbf{w}}) $. From the performance comparisons shown in Table~\ref{tab_onsi_comp}, it {\color{black}was} found that the proposed algorithms T1HFIT$^{\text{S}}$ and T1HFIT$^{\text{M}}$ were better than the T1FIS algorithms FALCON, SaFIN, and SONFIN. SONFIN offered the test RMSE $ E_t = 0.0085$ with the smallest parameter count $ c(\text{\textbf{w}}) = 36 $; whereas, the proposed algorithm T1HFIT$^{\text{M}}$ offered {\color{black}the} better test RMSE $ E_t = 0.0041$ with a slightly larger parameter count $ c(\text{\textbf{w}}) = 40$. 

Similarly, the proposed T2FIS algorithms T2HFIT$^{\text{S}}$ and T2HFIT$^{\text{M}}$ offered better performance compared to the algorithms T2FLS (Singleton), T2FLS (TSK), TSCIT2FNN, T2TSKFNS, T2FNN, SIT2FNN, RIT2NFS-WB, MRI2NFS. The algorithm SEIT2FNN reported test RMSE $ E_t = 0.0022$ and the parameter count was 84; whereas, in comparison to SEIT2FNN, the algorithm T2HFIT$^{\text{M}}$ offered a slightly higher test RMSE $ E_t = 0.0028 $, but had {\color{black}a} lower parameter count $ c(\text{\textbf{w}}) ${\color{black},} i.e., 72. 

{\color{black}The time comparison, however, is limited since the training time depends on several factors: 1) the type of programming language used; 2) the platform and its configurations {\color{black}on which} programs were executed; 3) the way data were fed for the training; 3) the status of the cache memory (in the case CPU time is observed); etc. It may be noted that{\color{black},} from the available training time reported in the literature{\color{black},} the MRI2NFS, RIT2NFS-W, T2FLS-G, SEIT2FNN approximately takes 0.15, 0.17, 2.41, {\color{black}and} 2.24 minutes (CPU time only) respectively. {\color{black}This is} in comparison to T1HFIT$^{\text{M}}$ and T2HFIT$^{\text{M}}${\color{black}, which take} 1.52 and 3.23 minutes (including CPU time, other file operations, etc.) respectively. Since the training times taken by the algorithms were {\color{black}close} to {\color{black}one another} and {\color{black}the time comparison} has limitations{\color{black}, it} may be concluded that {\color{black}the proposed} models performed {\color{black}efficiently}. {\color{black}This is also} evident from the performance statistics given in Table~\ref{tab_onsi_stat} and the performance comparison is provided in Table~\ref{tab_onsi_comp}.}

The best models obtained using the proposed algorithms are illustrated in Fig.\ref{fig_onsi}, which shows the hierarchical structure of the derived models and the selected inputs are indicated by $ x_i $ in the models. The rectangular blocks in Fig~\ref{fig_onsi} show the nodes (a T1FIS or T2FIS) of the tree (hierarchical structure). The target and predicted value {\color{black}plots} of {\color{black}200} samples are shown in Fig.~\ref{fig_onsi_plot}.

\begin{figure}
	\centering
	\subfigure[T1HFIT$^{\text{S}}$: $ E_n = 0.0043$]
	{
		\includegraphics[width=0.23\columnwidth]{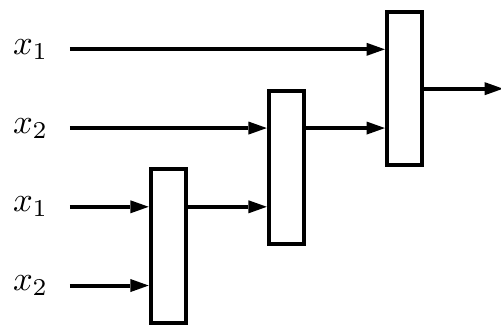}%
		\label{fig_onsi_1}%
	}
	\subfigure[T1HFIT$^{\text{M}}$: $ E_n = 0.0041$]
	{
		\includegraphics[width=0.23\columnwidth]{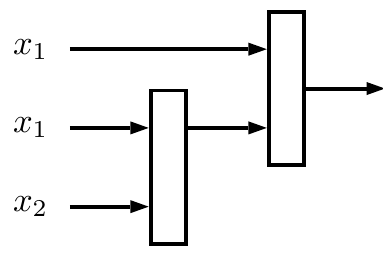}%
		\label{fig_onsi_2}%
	}
	\subfigure[T2HFIT$^{\text{S}}$: $ E_n = 0.0034$]
	{
		\includegraphics[width=0.23\columnwidth]{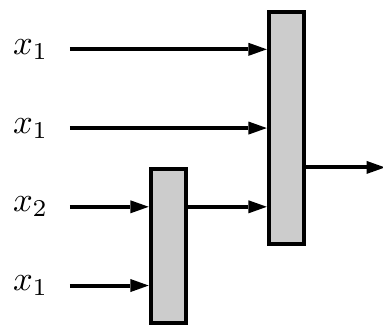}%
		\label{fig_onsi_3}%
	}
	\subfigure[T2HFIT$^{\text{M}}$: $ E_n = 0.0028$]
	{
		\includegraphics[width=0.23\columnwidth]{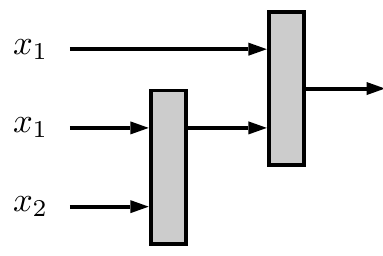}%
		\label{fig_onsi_4}%
	}
	\caption{Example--1: designed HFIT, where the shaded nodes indicate T2FIS.}
	\label{fig_onsi}
\end{figure}
\begin{figure}
	\centering
	\includegraphics[width=0.6\columnwidth]{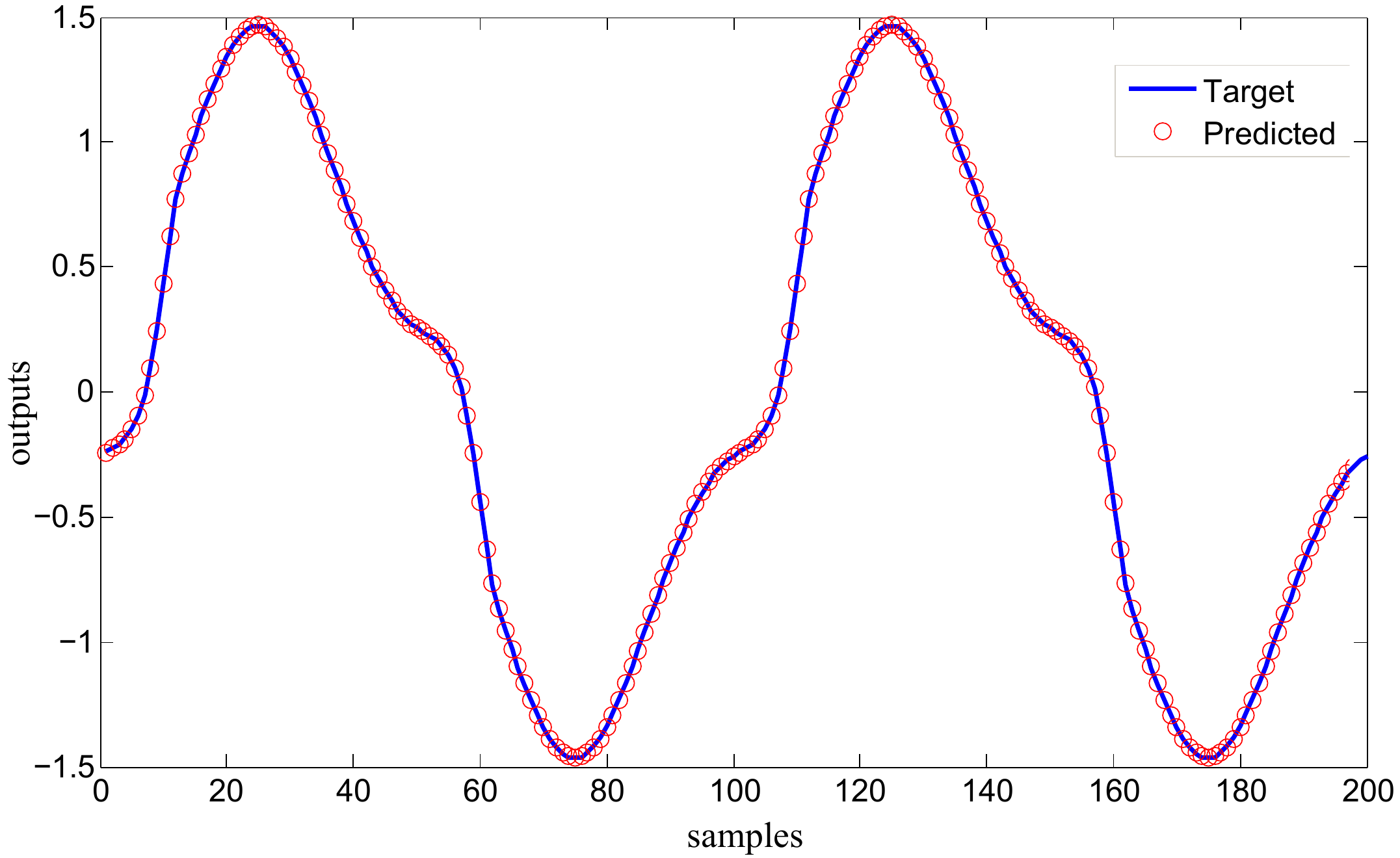}%
	\caption{Example--1: target versus predicted test values. The test outputs belong to algorithm T2HFIT$^{\text{M}}${\color{black}, which} has the test RMSE $ E_t =$ 0.0028.}
	\label{fig_onsi_plot}
\end{figure}

\subsection{Example 2---Noisy Chaotic Time Series Prediction}
\label{sec_fis_exp_2}
\subsubsection{Case--Clean Set}
\label{sec_perf_ex2}
A chaotic time series dataset, the Mackey-Glass chaotic time series, was used in this example, which was generated using the following delay differential equation:
\begin{equation}
\label{eq_mgts}
\frac{dx(k)}{dk} = \frac{0.2x(k-\tau)}{1+x^{10}(k-\tau)}-0.1x(k)
\end{equation}
where delay constant $ \tau > 17 $ and {\color{black}$k$ is the time step{\color{black}.}} In this example, the objective was to predict $x(k)$ using the past outputs of the time series as mentioned in~\cite{juang2010interval}. Hence, {\color{black}the} input--output pattern was of the form:
$$ \left[ x(k-24), x(k-18), x(k-12), x(k-6); x(k)\right]. $$
Let us say {\color{black}that} the inputs are $ x_1  = x(k-24)$ , $ x_2  = x(k-18)$, $ x_3  = x(k-12)$, and $ x_4  = x(k-6)$. {\color{black}For the training of the proposed algorithms, a total of 1000 patterns were generated from $ k=124$ to $1123${\color{black},} with the parameter $ \tau $ being set to 30 and $ x(0) $ being set to 1.2~\cite{juang2010interval}.} This set of training patterns were clean (no noise {\color{black}was} added). From the generated clean patterns, {\color{black}as mentioned in~\cite{juang2010interval},} the first 500 patterns (clean training set) were used for training {\color{black}purposes} and the second 500 patterns (clean test set) were used for {\color{black}test} purposes. {\color{black}Aiming to assess the average performance of the proposed algorithms, ten} repetitions of training and {\color{black}testing} were performed using clean training and test sets, and the results were collected accordingly (Table~\ref{tab_mcg_stat_clean}). Table~\ref{tab_mcg_com_clean} shows the comparison of results {\color{black}of} the proposed algorithms {\color{black}(the best among ten models)} with the {\color{black}best} results reported by {\color{black}the} algorithms listed in Table~\ref{tab_lit_fis}.

For this example (clean set), the performance statistics {\color{black}are} shown in Table~\ref{tab_mcg_stat_clean}. The obtained statistics illustrate that the proposed algorithms T1HFIT$^{\text{S}}$, T1HFIT$^{\text{M}}$, T2HFIT$^{\text{S}}$, and T2HFIT$^{\text{M}}$ performed with high accuracies. It shows that the mean correlation coefficient $ r_n $ of {\color{black}the} training {\color{black}set of all algorithms} is 1.00, and the mean correlation coefficient $ r_t $ of {\color{black}the} test set of the algorithms T1HFIT$^{\text{S}}$, T1HFIT$^{\text{M}}$, T2HFIT$^{\text{S}}$, and T2HFIT$^{\text{M}}$ are 0.9858, 0.9864, 0.9783, and 0.9912 respectively. That is the test correlation coefficients are closer to 1.00 ({\color{black}a} high positive correlation between target and predicted outputs). Such performance indicates that the {\color{black}algorithms} consistently performed with a high accuracy, and the obtained low values of STDs are {\color{black}evidence} of this fact (Table~\ref{tab_mcg_stat_clean}). 

Moreover, the Pareto-based  multiobjective {\color{black}training} offered less complex models (the mean parameter count $ c(\text{\textbf{w}}) $ of T1HFIT$^{\text{M}}$ was 57.6 compared to 71.6 of T1HFIT$^{\text{S}}$ and {\color{black}the} $ c(\text{\textbf{w}}) $ of T2HFIT$^{\text{M}}$ was 129.5 compared to 203.4 of T1HFIT$^{\text{S}}$) with high accuracies (Table~\ref{tab_mcg_stat_clean}). Hence, like {\color{black}in} example 1, in this example also {\color{black}the} Pareto-based multiobjective was advantageous to use, which provided the option to choose the best solution from a Pareto-front. Fig.~\ref{fig_pareto_front} {\color{black}illustrates} a Pareto-front {\color{black}created} during the multiobjective training of HFIT. 

\begin{table}
	\centering
	\caption{Performance Evaluation on Clean Set of Noisy Chaotic Time Series Prediction (Example-2)}
	\subtable[Performance Statistics (10 repetitions)]{
	\setlength{\tabcolsep}{3pt}
	\label{tab_mcg_stat_clean}
	\begin{tabular}{llrrrr}
		\toprule
		&  & T1HFIT$^{\text{S}}$ & T1HFIT$^{\text{M}}$ & T2HFIT$^{\text{S}}$ & T2HFIT$^{\text{M}}$ \\
		\hline
		$E_n$ & Best & 0.0115 & 0.0115 & 0.0108 & 0.0032 \\
		& Mean & 0.0345 & 0.0338 & 0.0413 & 0.0224 \\
		& STD & 0.0163 & 0.0207 & 0.0221 & 0.0203 \\
		$r_t$ & Best & 1.00 & 1.00 & 1.00 & 1.00 \\
		& Mean & 0.9858 & 0.9864 & 0.9783 & 0.9912 \\
		& STD & 0.0117 & 0.0107 & 0.0182 & 0.0154 \\
		$E_t$ & Best & 0.0122 & 0.0119 & 0.0086 & 0.0058 \\
		& Mean & 0.0414 & 0.0356 & 0.0427 & 0.0275 \\
		& STD & 0.0224 & 0.0173 & 0.0234 & 0.0207 \\
		$r_t$ & Best & 1.00 & 1.00 & 1.00 & 1.00 \\
		& Mean & 0.9786 & 0.9850 & 0.9769 & 0.9888 \\
		& STD & 0.0211 & 0.0120 & 0.0195 & 0.0158 \\
		$ c(\text{\textbf{w}}) $ & Best & 20 & 40 & 72 & 36 \\
		& Mean & 71.6 & 57.6 & 203.4 & 129.5 \\
		{\color{black}Time} & Best & 8.42 & 5.51 & 21.33 & 7.91 \\
			 & Mean & 71.6 & 11.03 & 31.83 & 16.58 \\
		\bottomrule
	\end{tabular}}\quad~\quad
	\subtable[Performance Comparison]{
		\setlength{\tabcolsep}{5pt}
			\label{tab_mcg_com_clean}
			\begin{tabular}{llccr}
				\toprule
				& Algorithm & $ E_n $ & $ E_t $ & $ c(\text{\textbf{w}}) $ \\
				\hline \multirow{11}{*}{\begin{sideways} {Type--1}\end{sideways}}
				& NNT1FW & $-$ & 0.0550 & $-$ \\
				& AFRS & 0.0267 & 0.0256 & 78 \\
				& IFRS & 0.0240 & 0.0253 & 58 \\
				& H-TS-FS$ ^1 $ & 0.0120 & 0.0129 & 148 \\
				& H-TS-FS$ ^2 $ & 0.0145 & 0.0151 & 46 \\
				& RBF-AFA & $-$ & 0.0128 & $-$ \\
				& HyFIS & $-$ & 0.0100 & $-$ \\
				& D-FNN & $-$ & 0.0080 & 70$^*$ \\
				& SuPFuNIS & $-$ & 0.0057 & 70$^*$ \\
				& \textbf{T1HFIT$^{\text{S}}$} & 0.0115 & 0.0122 & 60 \\
				& \textbf{T1HFIT$^{\text{M}}$} & 0.0115 & 0.0119 & 40 \\
				\hline \multirow{7}{*}{\begin{sideways} {Type--2}\end{sideways}}
				& T2FLS (singleton) & $-$ & 0.0426 & $-$ \\
				& T2FLS (TSK) & $-$ & 0.0431 & $-$ \\
				& NNT2FW & $-$ & 0.0390 & $-$ \\
				& SEIT2FNN$ ^1 $ & $-$ & 0.0034 & 126$^*$ \\
				& SEIT2FNN$ ^2 $ & $-$ & 0.0053 & 90$^*$ \\
				& \textbf{T2HFIT$^{\text{S}}$} & 0.0108 & 0.0086 & 108 \\
				& \textbf{T2HFIT$^{\text{M}}$} & 0.0032 & 0.0058 & 118 \\
				\bottomrule  
				\multicolumn{5}{l}{$^*$This is approximately calculated. It may be larger.}
			\end{tabular}
	}
\end{table}

Table~\ref{tab_mcg_com_clean} {\color{black}describes} the comparison between several {\color{black}algorithms} on clean training and test set. The {\color{black}results} of IFRS, AFRS, H-TS-FS1, and H-TS-FS2 {\color{black}were} collected from~\cite{chen2007automatic}; RBF-AFA, HyFIS, D-FNN, and SuPFuNIS from~\cite{juang2008self}; NNT1FW and NNT2FW from~\cite{angelov2004approach}; and T2FLS (Singleton), T2FLS (TSK), and SEIT2FN from~\cite{juang2008self}.

The training and test RMSEs and the {\color{black}parameter} count $ c(\text{\textbf{w}}) $ were used for comparing the algorithms, which is shown in Table~\ref{tab_mcg_com_clean}. {\color{black}A training time comparison for this example cannot be performed because of the unavailability of the training time of other algorithms in the literature. 

In T1FIS comparisons, it was found that the proposed algorithms T1HFIT$^{\text{S}}$ and T1HFIT$^{\text{M}}$ performed better than or {\color{black}were} competitive {\color{black}with} the algorithms NNT1FW, AFRS, IFRS, H-TS-FS, RBF-AFA, and HyFIS. The algorithms D-FNN and SuPFuNIS had better test RMESs $ E_t = 0.008$ and $ E_t = 0.005 $, but their parameter counts were larger since the number of rules in each case was 10. Since each T1FS MF has at least two parameters and each rule has three free parameters at the consequent part, the number of parameter count for two input variables stands to at least 70 (this is {\color{black}an} approximate calculation since D-FNN and SuPFuNIS may have other parameters that may increase {\color{black}the} parameter count value). Whereas, the algorithms T1HFIT$^{\text{S}}$ and T1HFIT$^{\text{M}}$ had parameter counts equal to 60 and 40{\color{black},} respectively. Therefore, T1HFIT$^{\text{S}}$ and T1HFIT$^{\text{M}}$ are {\color{black}competitive} {\color{black}with  D-FNN and SuPFuNIS}{\color{black}.}}

In T2FIS, the proposed algorithms clearly performed better than T2FLS (Singleton), T2FLS (TSK), and NNT2FW. Whereas, the performance of the proposed algorithms were competitive with SEIT2FNN$ ^1 $ (without fuzzy set reduction) and SEIT2FNN$ ^2 $ (with fuzzy set reduction) whose test RMSEs $ E_t $ were 0.0034 and 0.0058, respectively. The algorithm SEIT2FNN$ ^1 $ had 28 fuzzy sets and SEIT2FNN$ ^2 $ had 16 fuzzy sets (reduced), and each of these had seven rules. Hence, the parameter count of these algorithms stands to at least 126 and 90, respectively. On the other hand, the proposed algorithm T2HFIT$^{\text{S}}$ had {\color{black}a} test RMSE {\color{black}of} $ E_t = 0.0086 $ (slightly larger than SEIT2FNN$ ^1 $ and SEIT2FNN$ ^2 $), but the parameter count was 108, which is smaller than SEIT2FNN$ ^1 $. Similarly, the proposed algorithm T2HFIT$^{\text{M}}$ had {\color{black}a} test RMSE {\color{black}of} $ E_t = 0.0058$, which is  {\color{black}close} to SEIT2FNN$ ^2 $ and the {\color{black}parameter} count was smaller than SEIT2FNN$ ^1 $ and closer to SEIT2FNN$ ^2 $. Therefore, {\color{black}in this case,} the proposed T2HFIT$^{\text{M}}$ is as efficient {\color{black}as SEIT2FNN$ ^1 $ and SEIT2FNN$ ^2 $ are}. Fig.\ref{fig_mgts} shows the hierarchical structure of the best models obtained by the proposed algorithms. Additionally, the target versus prediction plot of test data samples is illustrated {\color{black}in Fig.~\ref{fig_mgts_plot}}. 
\begin{figure}
	\centering
	\subfigure[T1HFIT$^{\text{S}}$: $ E_n = 0.0115$]
	{
		\includegraphics[width=0.23\columnwidth]{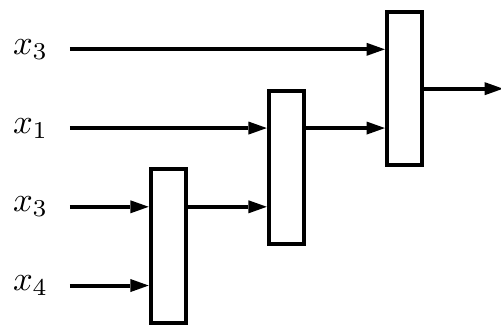}%
		\label{fig_mgts_1}%
	}
	\subfigure[T1HFIT$^{\text{M}}$: $ E_n = 0.0115$]
	{
		\includegraphics[width=0.23\columnwidth]{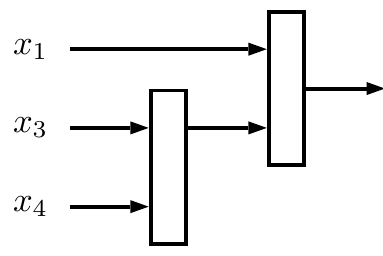}%
		\label{fig_mgts_2}%
	}
	\subfigure[T2HFIT$^{\text{S}}$: $ E_n = 0.0108$]
	{
		\includegraphics[width=0.23\columnwidth]{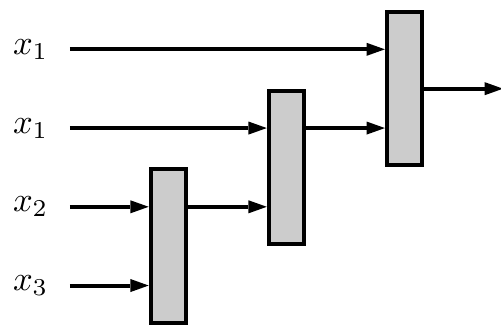}%
		\label{fig_mgts_3}%
	}
	\subfigure[T2HFIT$^{\text{M}}$: $ E_n = 0.0032$]
	{
		\includegraphics[width=0.23\columnwidth]{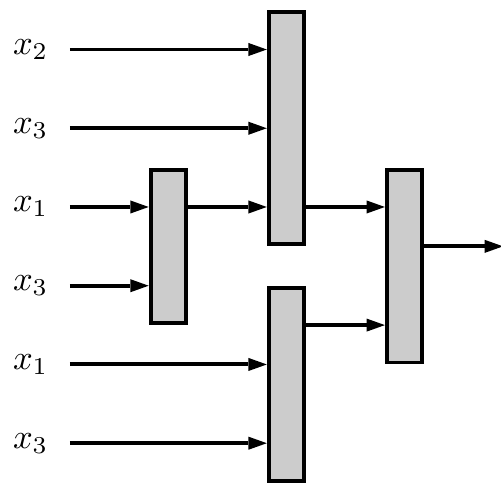}%
		\label{fig_mgts_4}%
	}
	\caption{Example--2 Clean set: designed HFIT. The shaded nodes are T2FIS. }
	\label{fig_mgts}
\end{figure}
\begin{figure}
	\centering
	\includegraphics[width=0.6\columnwidth]{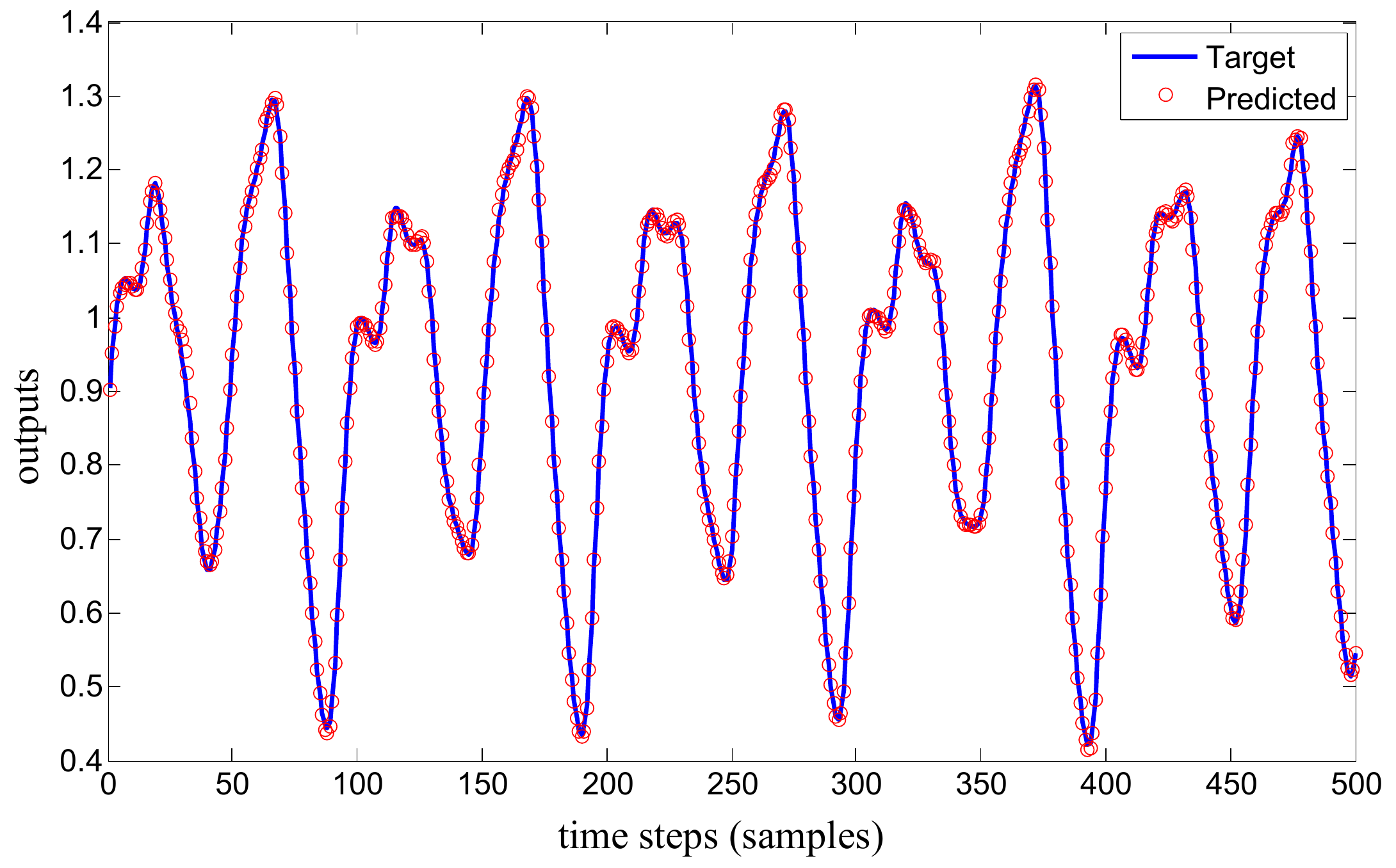}%
	\caption{Example--2 Clean set: target versus predicted test values. The test outputs belong to algorithm T2HFIT$^{\text{M}}$ that has the test RMSE $ E_t =$ 0.0058.}
	\label{fig_mgts_plot}
\end{figure}

\subsubsection{Case--Noisy Set}
The performances of the proposed algorithms were further evaluated for {\color{black}noisy} patterns. Therefore, three training sets and three test sets were generated by adding Gaussian noise with a mean {\color{black}of} 0 and STDs of 0.1, 0.2, and 0.3 to the original data $ x(k) $ as described in~\cite{juang2010interval}. These noisy training sets (with STDs 0.1, 0.2, and 0.3) were presented for the training of the proposed algorithms. With each training set of STDs {\color{black}of} 0.1, 0.2, and 0.3, three test sets were given for testing: clean, STD 0.1, and STD 0.3. The obtained results were compared with the results reported in the literature (Table~\ref{tab_mcg_com_noise}).  
\begin{table*}
	\centering
	\setlength{\tabcolsep}{3pt}
	{\footnotesize
	\caption{Example 2--Noisy set: Performance Comparison}
	\label{tab_mcg_com_noise}
	\begin{tabular}{llccccrccccrccccr}
		\toprule
		& & Train & \multicolumn{4}{c}{Test}  & Train & \multicolumn{4}{c}{Test} & Train & \multicolumn{4}{c}{Test}   \\
		\cline{4-7}\cline{9-12}\cline{14-17} 
		FIS & Algorithm &  0.1 & clean &  0.1 &  0.3 & $ c(\text{\textbf{w}}) $ &  0.2 & clean &  0.1 &  0.3 & $ c(\text{\textbf{w}}) $ &  0.3 & clean &  0.1 &  0.3 & $ c(\text{\textbf{w}}) $  \\
		\hline \multirow{6}{*}{\begin{sideways} {Type--1}\end{sideways}} 
		& SVR-FM & 0.128 & 0.045 & 0.087 & 0.200 & 1127 & 0.229 & 0.089 & 0.109 & 0.189 & 1127 & 0.332 & 0.138 & 0.147 & 0.198 & 1127 \\
		& EFuNN & 0.126 &  $-$ &  $-$ &  $-$ &  $-$ & 0.252 &  $-$ &  $-$ &  $-$ &  $-$ & 0.366 &  $-$ &  $-$ &  $-$ &  $-$ \\
		& DENFIS & 0.116 &  $-$ &  $-$ &  $-$ &  $-$ & 0.214 &  $-$ &  $-$ &  $-$ &  $-$ & 0.306 &  $-$ &  $-$ &  $-$ &  $-$ \\
		& SONFIN & 0.113 & 0.054 & 0.108 & 0.256 & 130 & 0.226 & 0.116 & 0.138 & 0.280 & 130 & 0.302 & 0.195 & 0.208 & 0.305 & 130 \\
		& \textbf{T1HFIT$^{\text{S}}$} & 0.127 & 0.050 & 0.140 & 0.363 & 60 & 0.234 & 0.111 & 0.153 & 0.349 & 104 & 0.305 & 0.100 & 0.159 & 0.356 & 64 \\
		& \textbf{T1HFIT$^{\text{M}}$} & 0.128 & 0.042 & 0.138 & 0.357 & 40 & 0.225 & 0.085 & 0.145 & 0.360 & 84 & 0.307 & 0.119 & 0.162 & 0.351 & 60 \\
		\hline \multirow{7}{*}{\begin{sideways} {Type--2}\end{sideways}} 
		& T2FLS-G & 0.133 & 0.074 & 0.103 & 0.220 & 110 & 0.238 & 0.125 & 0.132 & 0.200 & 110 & 0.357 & 0.232 & 0.234 & 0.264 & 110 \\
		& IT2FNN-SVR(N) & 0.128 & 0.048 & 0.087 & 0.193 & 103 & 0.234 & 0.085 & 0.105 & 0.186 & 103 & 0.349 & 0.127 & 0.138 & 0.188 & 103 \\
		& IT2FNN-SVR(F) & 0.127 & 0.046 & 0.088 & 0.215 & 103 & 0.233 & 0.083 & 0.103 & 0.180 & 103 & 0.347 & 0.121 & 0.131 & 0.184 & 103 \\
		& SEIT2FNN & 0.123 & 0.049 & 0.097 & 0.212 & 110 & 0.225 & 0.083 & 0.113 & 0.228 & 110 & 0.319 & 0.196 & 0.197 & 0.254 & 110 \\
		& eT2FIS & 0.120 & 0.059 & 0.107 & 0.214 & $-$ & 0.225 & 0.083 & 0.132 & 0.247 & $-$ & 0.327 & 0.102 & 0.152 & 0.278 & $-$ \\
		& \textbf{T2HFIT$^{\text{S}}$} & 0.128 & 0.039 & 0.135 & 0.355 & 108 & 0.227 & 0.079 & 0.143 & 0.348 & 82 & 0.314 & 0.100 & 0.148 & 0.354 & 144 \\
		& \textbf{T2HFIT$^{\text{M}}$} & 0.123 & 0.042 & 0.135 & 0.365 & 72 & 0.233 & 0.087 & 0.144 & 0.348 & 72 & 0.311 & 0.097 & 0.148 & 0.356 & 108 \\
		\bottomrule
	\end{tabular}}
\end{table*}

Table~\ref{tab_mcg_com_noise} describes the comparisons between the results of the algorithms, where the results of SONFIN and SVR-FM were collected from~\cite{juang2010interval}, DENFIS and EFuNN from~\cite{tung2013et2fis}, SEIT2FNN, $\text{T2FLS-G}$, IT2FNN-SVR(N), IT2FNN-SVR(F) from~\cite{juang2010interval}, and eT2FIS from~\cite{lin2014simplified}. It is evident from the comparison of the results that the proposed algorithms performed efficiently over the noisy datasets and the obtained models were less complex than the other models listed in Table~\ref{tab_mcg_com_noise}. Particularly when T1FISs were compared.  Moreover, for each noisy data (STD 0.1, STD 0.2, and STD 0.3), the proposed algorithms had {\color{black}a} smaller parameter count and had a lower or {\color{black}competitive} training RMSE $ E_n $ compared to other listed algorithms. In T1FIS comparisons, the SONFIN had {\color{black}a} slightly better RMSE, but the number of parameters counts was larger than the proposed algorithms T1HFIT$^{\text{S}}$ and T1HFIT$^{\text{M}}$. Similarly, in T2FIS comparison, the algorithm eT2FIS had slightly better RMSE than the other listed algorithms, but the models obtained using the proposed algorithms were less complex, i.e., had a smaller parameter count.

\subsection{Example 3---Miles-Per-Gallon Prediction Problem}
To evaluate the performance of the proposed algorithms, a real-world MPG problem was used. The objective of this example was to predict or estimate the city-cycle fuel consumption in MPG. The MPG dataset was collected from the UCI machine learning repository~\cite{UCILichman2013}. This dataset has 392 {\color{black}samples,} each of which has six input variables, but in this example, as mentioned in~\cite{das2015evolving}, three variables ($ x_1 =$ weight, $ x_2 =$ acceleration, and $ x_3 =$ model year) were selected.  In the training process, 50\% (196 patterns) {\color{black}of} samples were randomly selected for training and the rest of the 50\% (196 patterns) {\color{black}of} samples were taken for testing. Such {\color{black}a} process {\color{black}for} the training set and {\color{black}test} set selection was repeated ten times. {\color{black}Accordingly, the collected} performance statistics {\color{black}are} shown in Table~\ref{tab_mpg_stat}. 

The performances of the proposed algorithms were compared with the literature (Table~\ref{tab_mpg_stat}). However, the algorithms chosen from the literature were tested over fewer test samples. Therefore, the comparison {\color{black}shown} in Table~\ref{tab_mpg_stat} {\color{black}is} limited to the comparison of the training RMSE because all the mentioned algorithms were trained over the same number of training samples. For the comparisons, the {\color{black}T1FLS result} was collected from~\cite{juang2013reduced} and the results of SEIT2FNN, RIT2NFS-WB, McIT2FIS-UM, and McIT2FIS-US were collected from~\cite{das2015evolving}. 

\begin{table}
	\centering
	\caption{Performance Evaluation on Miles-Per-Gallon Prediction Problem (Example-3)}
	\subtable[Performance Statistics (10 repetitions)]{
	\setlength{\tabcolsep}{1.5pt}
	\label{tab_mpg_stat}
	\begin{tabular}{llrrrr}
		\toprule
		&  & T1HFIT$^{\text{S}}$ & T1HFIT$^{\text{M}}$ & T2HFIT$^{\text{S}}$ & T2HFIT$^{\text{M}}$ \\
		\hline
		$E_n$ & Best & 1.8931 & 2.2686 & 2.0881 & 1.9582 \\
		& Mean & 2.7115 & 2.6037 & 2.4699 & 2.4052 \\
		& STD & 0.5144 & 0.4071 & 0.4461 & 0.3774 \\
		$r_n$ & Best & 0.97 & 0.96 & 0.96 & 0.96 \\
		& Mean & 0.921 & 0.941 & 0.946 & 0.950 \\
		& STD & 0.1035 & 0.0218 & 0.0244 & 0.0160 \\
		$E_t$ & Best & 2.7550 & 2.7907 & 2.8383 & 2.6623 \\
		& Mean & 4.2333 & 3.3349 & 3.4006 & 3.3172 \\
		& STD & 0.5024 & 0.5720 & 0.7423 & 0.6855 \\
		$r_t$ & Best & 0.97 & 0.96 & 0.96 & 0.96 \\
		& Mean & 0.921 & 0.941 & 0.946 & 0.950 \\
		& STD & 0.1035 & 0.0218 & 0.0244 & 0.0160 \\
		$ c(\text{\textbf{w}}) $ & Best & 20 & 20 & 108 & 118 \\
		& Mean & 132 & 78.8 & 224 & 207.4 \\
		{\color{black}Time} & Best & 1.89  & 1.75 & 8.75 & 8.53 \\
		     & Mean & 12.04 & 6.75 & 14.91 & 12.33 \\
		\bottomrule
	\end{tabular}}~\quad
\subtable[Performance Comparison (10 Repetitions)]{
		\setlength{\tabcolsep}{1.5pt}
		\label{tab_mpg_comp}
		\begin{tabular}{llccccc}
			\toprule
			 & Algorithm & Mean $E_n$ & STD & Mean $E_t$ & STD & \makecell{\color{black} Samples \\ \color{blue}(train, test)} \\
			 \hline \multirow{3}{*}{\begin{sideways} {Type--1}\end{sideways}} 
			 & T1FLS & $-$ & $-$ & 3.5960 & $-$ & 196, 120 \\
			 & \textbf{T1HFIT$^{\text{S}}$} & 2.7115 & 0.5144 & 4.2333 & 0.5024 & 196, 196 \\
			 & \textbf{T1HFIT$^{\text{M}}$} & 2.6037 & 0.4071 & 3.3349 & 0.5720 & 196, 196 \\
			 \hline \multirow{6}{*}{\begin{sideways} {Type--2}\end{sideways}} 
			 & McIT2FIS-US & 2.7358 & $-$ & 2.6770 & $-$ & 196, 120 \\
			 & SEIT2FNN & 2.7161 & $-$ & 2.7895 & $-$ & 196, 120  \\
			 & McIT2FIS-UM & 2.6524 & $-$ & 2.6486 & $-$ & 196, 120 \\
			 & RIT2NFS-WB & 2.3685 & $-$ & 2.7807 & $-$ & 196, 120  \\
			 & \textbf{T2HFIT$^{\text{S}}$} & 2.4699 & 0.4461 & 3.4006 & 0.7423 & 196, 196  \\
			 & \textbf{T2HFIT$^{\text{M}}$} & 2.4052 & 0.3774 & 3.3172 & 0.6855 & 196, 196 \\
			\bottomrule
		\end{tabular}
}
\end{table} 

The comparisons of the models in this example were based on the mean training and test RMSEs $ E_n $ and $ E_t $ obtained for the ten repetitions. However, the comparison on test RMSEs was limited since only 120 {\color{black}samples} were used for testing by the algorithms considered from literature. Whereas, the algorithms proposed in this work used 196 samples for testing (Table~\ref{tab_mpg_comp}). It was observed that the proposed algorithms T2HFIT$^{\text{S}}$ and T2HFIT$^{\text{M}}$ outperformed all the other algorithms except {\color{black}for} RIT2NFS-WB, which had {\color{black}a} slightly better training RMSE $ E_n = 2.3685 $ in comparison to the training RMSEs $ E_n= 2.4699 $ and $ E_n=2.4052 $ of T2HFIT$^{\text{S}}$ and T2HFIT$^{\text{M}}${\color{black},} respectively. Since the performance comparisons were based on the average value of ten repetitions, the model's hierarchical structures are not presented for this example. 

{\color{black}From the available training time reported in the literature, it may be noted that the algorithms McIT2FIS-UM, McIT2FIS-US, RIT2NFS-WB, and SEIT2FNN take 0.0025, 0.003, 0.16, {\color{black}and} 0.33 minutes (CPU time only) compared to T2HFIT$^{\text{M}}$, which takes 8.23 minutes. However, it may be noted that T2HFIT$^{\text{M}}$ is a two-phase population-based learning algorithm, whereas {\color{black}the} other algorithms are {\color{black}single-solution} based algorithms.} 

\subsection{Example 4---Abalone Age Prediction}
In this example, a prediction problem was taken in which a person's age was predicted based on {\color{black}their} physical measurements. The {\color{black}Abalone} dataset was collected from the UCI machine learning repository~\cite{UCILichman2013}. It has 4177 data samples{\color{black},} {\color{black}each} of which has seven input variables ($ x_1 =$ length, $ x_2 =$ diameter, $ x_3 =$ height, $ x_4 =$ whole weight, $ x_5 =$ shucked weight, $ x_6 =$ viscera weight, and $ x_7 =$ shell weight) and one output variable (rings). To train the proposed algorithms, 80\% (3342 patterns) {\color{black}of} samples were randomly taken for training and 20\% (835 patterns) {\color{black}remaining} samples were taken for testing. {\color{black}Additionally, in this work, to assess the average performance of proposed algorithms,} training process was repeated ten times, and the collected results are summarized in Table~\ref{tab_abl_stat}. 

The obtained results are compared with the results reported in the literature (Table~\ref{tab_abl_com}). For the comparisons, {\color{black}the} results of General, HS, CCL, and Chen\&Cheng were collected from~\cite{juang2013reduced}, and the results of SEIT2FNN, RIT2NFS-WB, McIT2FIS-UM, and McIT2FIS-US were collected from~\cite{das2015evolving}. The algorithms General~\cite{baranyi2004generalized}, CCL~\cite{chang2008fuzzy}, HS~\cite{huang2008fuzzy}, and WFRI-GA~\cite{chen2011weighted} were fuzzy interpolate reasoning methods, where WFRI-GA was based on {\color{black}the} genetic algorithm and the algorithm `General' implemented {\color{black}the} Mamdani type FIS. It is evident from the results in Table~\ref{tab_abl_com} that the proposed algorithms (both T1FIS and T2FIS) {\color{black}outperformed} the algorithms considered {\color{black}for} comparisons. 

However, when comparing the test RMSEs, {\color{black}McIT2FIS-US}, McIT2FIS-UM, and RIT2NFS-WB had a slight edge over T2HFIT$^{\text{S}}$ and T2HFIT$^{\text{M}}$, but the parameter count of T2HFIT$^{\text{M}}$ was smallest among all {\color{black}the algorithms}, and it had the lowest training error. Hence, it may be concluded that {\color{black}T2HFIT$^{\text{M}}$} is the best performing algorithm for example 4. {\color{black} {\color{black}T2HFIT$^{\text{M}}$ performance falls behind only in training time comparison because} T2HFIT$^{\text{M}}${\color{black},} being a population based algorithm{\color{black},} takes longer training time than the other algorithms. The algorithms McIT2FIS-US, McIT2FIS-UM, RIT2NFS-WB, and SEIT2FNN take 1.81, 2.35, 5.48, {\color{black}and} 17.33 minutes (CPU time only) compared to T2HFIT$^{\text{M}}$, which takes 65.02 minutes. It is important to note that the other algorithms are single solution based algorithms. }The best-performing models of the proposed algorithms are illustrated in Fig.~\ref{fig_abl}, where the selected input feature is indicated by $ x_i $.

\begin{table}
	\centering
	\caption{Performance Evaluation on Abalone Age Prediction Problem (Example-4)}
	\subtable[Performance Statistics (10 Repetitions)]{
		\setlength{\tabcolsep}{3pt}
	\label{tab_abl_stat}
	\begin{tabular}{llrrrr}
		\toprule
		&  & T1HFIT$^{\text{S}}$ & T1HFIT$^{\text{M}}$ & T2HFIT$^{\text{S}}$ & T2HFIT$^{\text{M}}$ \\
		\hline
		$E_n$ & Best & 2.1097 & 2.2857 & 2.1154 & 2.1275 \\
		& Mean & 2.3267 & 2.4284 & 2.2597 & 2.2404 \\
		& STD & 0.1534 & 0.1079 & 0.1478 & 0.0627 \\
		$r_n$ & Best & 0.76 & 0.71 & 0.76 & 0.75 \\
		& Mean & 0.688 & 0.655 & 0.710 & 0.716 \\
		& STD & 0.0490 & 0.0347 & 0.0481 & 0.0204 \\
		$E_t$ & Best & 2.1260 & 2.3480 & 2.1824 & 2.1428 \\
		& Mean & 2.3644 & 2.4843 & 2.3808 & 2.3533 \\
		& STD & 0.1448 & 0.1029 & 0.1676 & 0.1127 \\
		$r_t$ & Best & 0.76 & 0.71 & 0.76 & 0.75 \\
		& Mean & 0.688 & 0.655 & 0.710 & 0.716 \\
		& STD & 0.0490 & 0.0347 & 0.0481 & 0.0204 \\
		$ c(\text{\textbf{w}}) $ & Best & 20 & 20 & 144 & 72 \\
		& Mean & 77.6 & 46.4 & 188.4 & 152.9 \\
		{\color{black}Time} & Best & 45.53 & 40.22 & 88.5 & 65.02 \\
		     & Mean & 83.33 & 70.16 & 213.5 & 192 \\
		\bottomrule
	\end{tabular}}~\quad~\quad
\subtable[Performance Comparison]{
	\setlength{\tabcolsep}{4pt}
	\label{tab_abl_com}
	\begin{tabular}{llccr}
		\toprule
		 & Algorithm & $E_n$ & $E_t$ & $ c(\text{\textbf{w}}) $ \\
		 \hline \multirow{6}{*}{\begin{sideways} {Type--1}\end{sideways}} 
		 & HS & $-$ & 3.1600 & $-$ \\
		 & General & $-$ & 3.1500 & $-$ \\
		 & CCL & $-$ & 2.6500 & $-$ \\
		 & Chen\&Cheng & $-$ & 2.5900 & $-$ \\
		 & \textbf{T1HFIT$^{\text{S}}$} & 2.1097 & 2.1260 & 124 \\
		 & \textbf{T1HFIT$^{\text{M}}$} & 2.2857 & 2.3480 & 84 \\
		 \hline \multirow{6}{*}{\begin{sideways} {Type--2}\end{sideways}} 
		 & RIT2NFS-WB & 2.4047 & 2.1346 & 131 \\
		 & McIT2FIS-UM & 2.3481 & 1.8740 & 115 \\
		 & SEIT2FNN & 2.3388 & 2.4330 & 140 \\
		 & McIT2FIS-US & 2.3357 & 1.8387 & 115 \\
		 & \textbf{T2HFIT$^{\text{S}}$} & 2.1154 & 2.1824 & 226 \\
		 & \textbf{T2HFIT$^{\text{M}}$} & 2.1275 & 2.1428 & 108 \\
		\bottomrule
	\end{tabular}
}
\end{table} 
\begin{figure}
	\centering
	\subfigure[T1HFIT$^{\text{S}}$: $ E_n = 2.1097$]
	{
		\includegraphics[width=0.23\columnwidth]{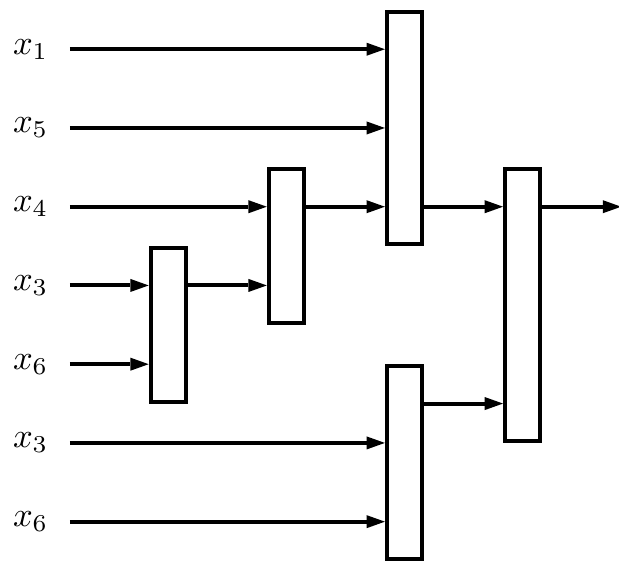}%
		\label{fig_abl_1}%
	}
	\subfigure[T1HFIT$^{\text{M}}$: $ E_n = 2.2857$]
	{
		\includegraphics[width=0.23\columnwidth]{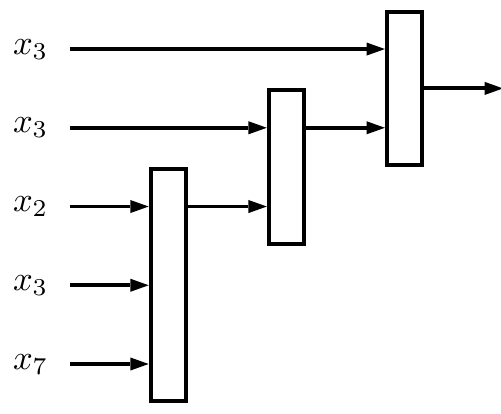}%
		\label{fig_abl_2}%
	}
	\subfigure[T2HFIT$^{\text{S}}$: $ E_n = 2.1154$]
	{
		\includegraphics[width=0.23\columnwidth]{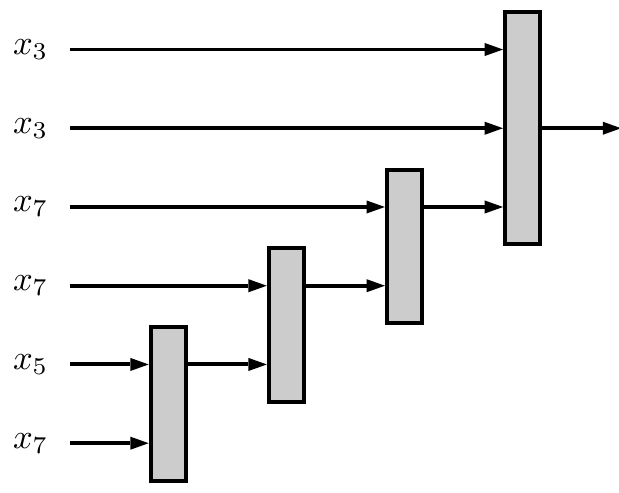}%
		\label{fig_abl_3}%
	}
	\subfigure[T2HFIT$^{\text{M}}$: $ E_n = 2.1275$]
	{
		\includegraphics[width=0.23\columnwidth]{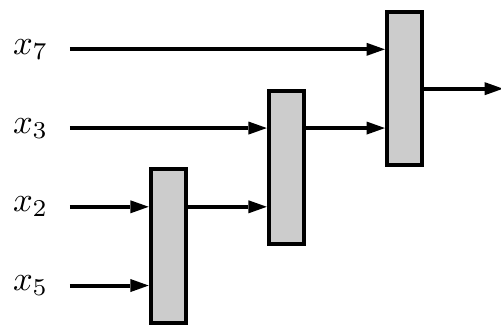}%
		\label{fig_abl_4}%
	}
	\caption{Example--4: designed HFIT, where the shaded nodes are T2FIS.}
	\label{fig_abl}
\end{figure}

\subsection{Example 5---Box-Jenkins Gas Furnace Problem}
In this example, the Box and Jenkins gas furnace dataset that was taken from~\cite{boxJenkins1976}, which has 296 data samples. The objective of this example was to predict the $ \text{CO}_2 $ concentration from the gas-flow rate. The {\color{black}gas} furnace system is modeled using a series, which is of the form: $ y(k) = f(y(k-1),u(k-4)$. For the training of the proposed models, {\color{black}as mentioned in~\cite{das2015evolving}}, 100\% (296 patterns) {\color{black}of the} samples were used. {\color{black} To show an average performance ability of the proposed algorithms, the} training process was {\color{black}also} repeated ten times, and the collected results are summarized in Table~\ref{tab_box_stat}. The performances of the proposed algorithms {\color{black}(the best results)} were compared with the {\color{black}best} performances of the algorithms reported in the literature (Table~\ref{tab_box_com}). 

To compare the performance of the algorithms, the results of T1-NFS and GNN were collected from~\cite{juang2013reduced}, and the results of SEIT2FNN, RIT2NFS-WB, McIT2FIS-UM, and McIT2FIS-US were collected from~\cite{das2015evolving}. As reported in Table~\ref{tab_box_com}, the proposed algorithms clearly outperformed the algorithms T1-NFS, GNN$ ^1 $, and GNN$ ^2 $ in the case of T1FIS comparisons and algorithms SEIT2FNN, RIT2NFS-WB, McIT2FIS-UM, and McIT2FIS-US in the case of T2FIS comparisons. 

For T2FIS, the proposed algorithm T2HFIT$^{\text{M}}$ provided a training RMSE $ E_n = 0.284 $, which was slightly lower than the training RMSE $ E_n = 0.269$ of SEIT2FNN. However, {\color{black}the} parameter count of T2HFIT$^{\text{M}}$ was 72 compared {\color{black}to 152} of SEIT2FNN. {\color{black}Additionally, {\color{black}despite} being a population based algorithm, {\color{black}T2HFIT$^{\text{M}}$} takes 6.31 minutes for the training{\color{black},} whereas SEIT2FNN takes 604.66 minutes for the training~\cite{juang2013reduced}.} Therefore, it may be concluded that{\color{black},} for example 5, {\color{black}T2HFIT$^{\text{M}}$} performed the best. The best-performing models are illustrated in Fig.~\ref{fig_box}.

\begin{table}
	\centering
	\caption{Performance Evaluation on Box-Jenkins Gas Concentration Problem (Example-5)}
	\subtable[Performance Statistics (10 Repetitions)]{
		\setlength{\tabcolsep}{2pt}
	\label{tab_box_stat}
	\begin{tabular}{llrrrr}
		\toprule
		&  & T1HFIT$^{\text{S}}$ & T1HFIT$^{\text{M}}$ & T2HFIT$^{\text{S}}$ & T2HFIT$^{\text{M}}$ \\
		\hline
		$E_n$ & Best & 0.246 & 0.280 & 0.256 & 0.275 \\
		& Mean & 0.303 & 0.344 & 0.291 & 0.301 \\
		& STD & 0.036 & 0.043 & 0.023 & 0.033 \\
		$r_n$ & Best & 0.97 & 0.97 & 0.97 & 0.97 \\
		& Mean & 0.959 & 0.947 & 0.963 & 0.960 \\
		& STD & 0.010 & 0.013 & 0.006 & 0.010 \\
		$ c(\text{\textbf{w}}) $ & Best & 40 & 40 & 72 & 72 \\
		& Mean & 132.8 & 58.4 & 286 & 167.4 \\
		{\color{black}Time} & Best & 5.41 & 4.22 & 9.82 & 6.31 \\
		     & Mean & 11.84 & 4.76 & 20.67 & 11.15 \\
		\bottomrule
	\end{tabular}}~\quad~\quad
	\subtable[Performance Comparison]{
		\setlength{\tabcolsep}{5pt}
	\label{tab_box_com}
	\begin{tabular}{llcr}
		\toprule
		& Algorithm & $E_n$ & $ c(\text{\textbf{w}}) $ \\
		\hline \multirow{5}{*}{\begin{sideways} {Type--1}\end{sideways}} 
		& T1-NFS & 0.4074 & $-$ \\
		& GNN$ ^1 $ & 0.3114 & $-$ \\
		& GNN$ ^2 $ & 0.2983 & $-$ \\
		& \textbf{T1HFIT$^{\text{S}}$} & 0.2455 & 124 \\
		& \textbf{T1HFIT$^{\text{M}}$} & 0.2838 & 40 \\
		\hline \multirow{6}{*}{\begin{sideways} {Type--2}\end{sideways}} 
		& SEIT2FNN & 0.2690 & 152 \\
		& RIT2NFS-WB & 0.3527 & 90 \\
		& McIT2FIS-UM & 0.3139 & 48 \\
		& McIT2FIS-US & 0.3181 & 48 \\
		& \textbf{T2HFIT$^{\text{S}}$} & 0.2767 & 154 \\
		& \textbf{T2HFIT$^{\text{M}}$} & 0.2840 & 72 \\
		\bottomrule
	\end{tabular}}
\end{table} 

\begin{figure}
	\centering
	\subfigure[T1HFIT$^{\text{S}}$: $ E_n = 0.2455$]
	{
		\includegraphics[width=0.23\columnwidth]{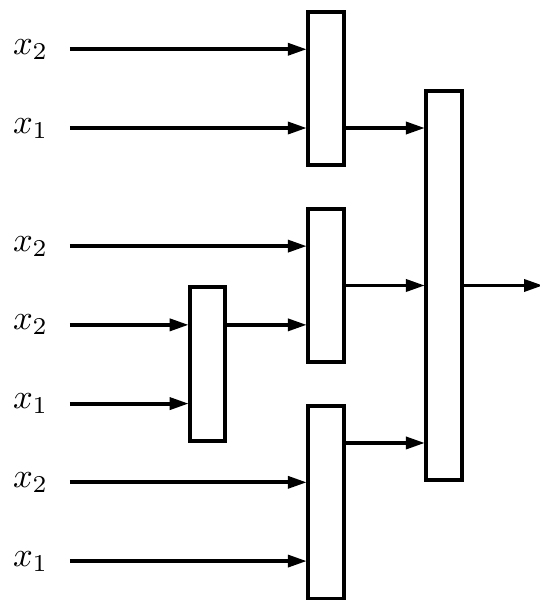}%
		\label{fig_box_1}%
	}
	\subfigure[T1HFIT$^{\text{M}}$: $ E_n = 0.2838$]
	{
		\includegraphics[width=0.23\columnwidth]{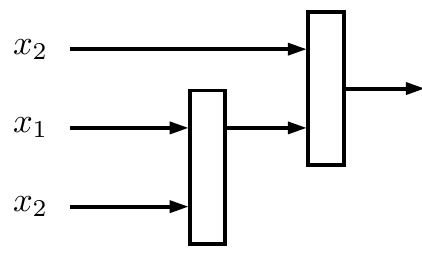}%
		\label{fig_box_2}%
	}
	\subfigure[T2HFIT$^{\text{S}}$: $ E_n = 0.2767$]
	{
		\includegraphics[width=0.23\columnwidth]{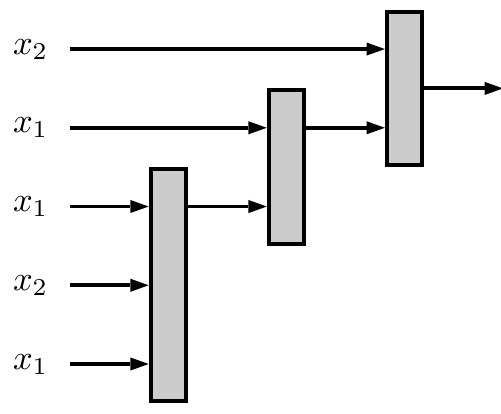}%
		\label{fig_box_3}%
	}
	\subfigure[T2HFIT$^{\text{M}}$: $ E_n = 0.2840$]
	{
		\includegraphics[width=0.23\columnwidth]{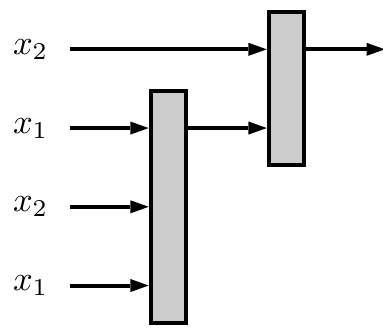}%
		\label{fig_box_4}%
	}
	\caption{Example--5: designed HFIT, where the shaded nodes are T2FIS.}
	\label{fig_box}
\end{figure}

\subsection{Example 6---Poly (lactic-co-glycolic acid) (PLGA) micro- and  {\color{black}nanoparticle} dissolution rate prediction}
\label{sec_perf_ex6}
This example illustrates a pharmaceutical {\color{black}industry} problem related to PLGA dissolution profile prediction, which is a complex problem since a vast number of factors governs its dissolution rate profile and it has a {\color{black}high} noise and redundancy because {\color{black}the} dataset was obtained from various experimental measurements and instruments. As per the dataset provided in~\cite{szlkek2013heuristic,ojha2015dimensionality}, this problem has 747 samples and a total of 300 input features, which influence the PLGA protein particle's dissolution rate~\cite{astete2006synthesis}. The input features are categorized into {\color{black}five groups}: protein descriptor, formulation characteristics, plasticizer, emulsifier, and time delay, which has 85, 17, 98, 99, and 1 features, respectively. 

{\color{black}The description of each feature group is as follows{\color{black}:} 1) The protein descriptors (85 features) {\color{black}describe} the type of molecules and proteins used in the {\color{black}drug's} manufacturing. 2) The formulation characteristics (17 features) describe the molecular properties{\color{black},} such as molecular weight, particle size, etc., of the molecules and proteins. 3) {\color{black}The} plasticizer (98 features) describes {\color{black}properties}{\color{black},} such as fluidity of the material used. 4) The emulsifier (99 features) describes the stabilizing properties of the material used in {\color{black}the drug's} manufacturing. 5) The time delay (1 feature) represents the time taken to dissolve/dissolute a sample drug.}

The PLGA dissolution profile prediction is a significant problem since it plays a crucial role in the medical application and toxicity evaluation of PLGA-based microparticles dosages~\cite{langer2004designing}. Moreover, PLGA microparticles are important diluents, which {\color{black}are} used for producing drugs in their correct dosage form. It is also used as a filler, as an excipient, and as an active pharmaceutical ingredient because it acts as a catalyst for drug absorption/dissolution/solubility{\color{black}~\cite{makadia2011poly}}. Therefore, PLGA dissolution is a widely studied research problem in pharmaceutical manufacturing and powder technology.

Using the parameter setting mentioned in Table~\ref{tab_parameter_fis} and using 10-fold cross-validation, the proposed algorithm T1HFIT$^{\text{M}}$ was able to select seven input features and was able to approximate a test RMSE of $ E_t = 18.66$. The selected features were: phase polyvinyl alcohol Mw ($ x_{90} $), ASA ($ x_{122} $), pH 8 msdon ($ x_{192} $), aromatic bond count ($ x_{204} $), a(xx) ($ x_{218} $), pH 12 msacc ($ x_{281} $), time days ($ x_{299} $). Similarly, the proposed algorithm T2HFIT$^{\text{M}}$ was able to approximate a test RMSE of $ E_t = 15.259$ with only four input features: aromatic atom count ($ x_{66} $), phase polyvinyl alcohol concentration inner phase ($ x_{88} $), pH 1 msdon ($ x_{285} $), time days ($ x_{299} $). {\color{black}Additionally, T2HFIT$^{\text{M}}$ provided a simple model (i.e., $ c(\text{\textbf{w}}) = 108$) compared to T2HFIT$^{\text{M}}$ that had {\color{black}} model complexity $ c(\text{\textbf{w}}) = 156$. Moreover, T2HFIT$^{\text{M}}$ takes 7.16 minutes of training time compared to the 45.7 minutes of T1HFIT$^{\text{M}}$. This difference in time is due to the difference between the number of input features being selected by T2HFIT$^{\text{M}}$ and T1HFIT$^{\text{M}}$.}

{\color{black}Feature} reduction is a significant task since it reduces {\color{black}drug's} manufacturing cost. Table~\ref{tab_plga_com} shows a comparison of the proposed T1HFIT$^{\text{M}}$ and T2HFIT$^{\text{M}}$ with {\color{black}algorithms} such as multilayer perceptron (MLP), reduced error pruning tree (REP Tree), heterogeneous flexible neural tree (HFIT), and Gaussian process regression (GPR). It is evident from the results that the proposed algorithm {\color{black}approximates} the PLGA dissolution profile with a lower number of features, and its approximation error was {\color{black}competitive} with the performance of other algorithms. Fig.~\ref{fig_plga} illustrates the obtained models for the prediction of {\color{black}the} PLGA dissolution profile.
\begin{table}
	\centering
	\caption{Example 6: Performance Comparison}
	\label{tab_plga_com}
	\begin{tabular}{llcr}
		\toprule
		Algorithm & Ref. & RMSE $ E_t $ & No. of features \\
		\hline
		MLP  & \cite{szlkek2013heuristic} & 14.3 & 17 \\
		HFIT & \cite{ojha2016ensemble} & 13.2 & 15 \\
		REP Tree  & \cite{ojha2015dimensionality} & 13.3 & 15 \\
		GPR  & \cite{ojha2015dimensionality} & 14.9 & 15 \\
		MLP  & \cite{ojha2015dimensionality} & 15.2 & 15 \\
		MLP  & \cite{szlkek2013heuristic} & 15.4 & 11 \\
		\textbf{T1HFIT$^{\text{M}}$} & present work & 18.6 & 7 \\
		\textbf{T2HFIT$^{\text{M}}$} & present work & 15.2 & 4 \\
		\bottomrule
	\end{tabular}
\end{table} 

\begin{figure}
	\centering
	\subfigure[T1HFIT$^{\text{M}}$: $ E_t = 18.66$]
	{
		\includegraphics[width=0.25\columnwidth]{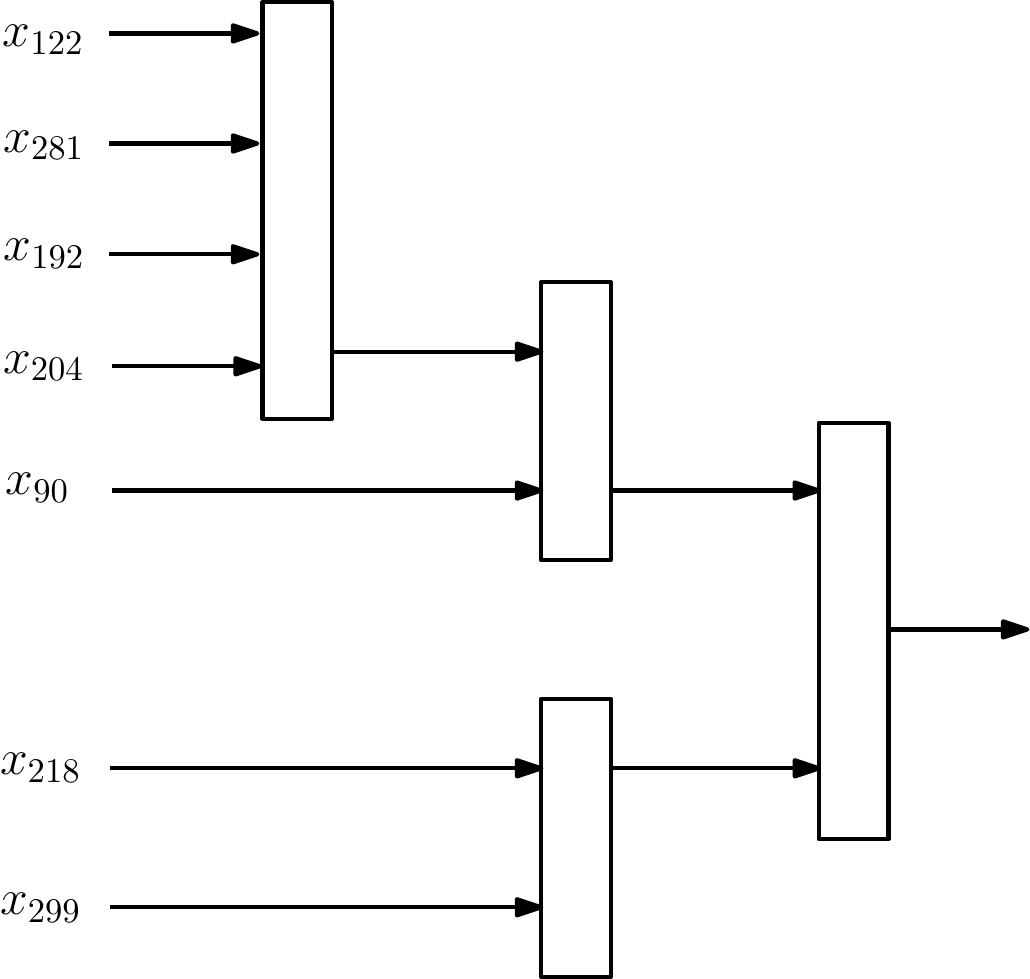}%
		\label{fig_plga_1}%
	}
	\subfigure[T2HFIT$^{\text{M}}$: $ E_t = 15.25$]
	{
		\includegraphics[width=0.25\columnwidth]{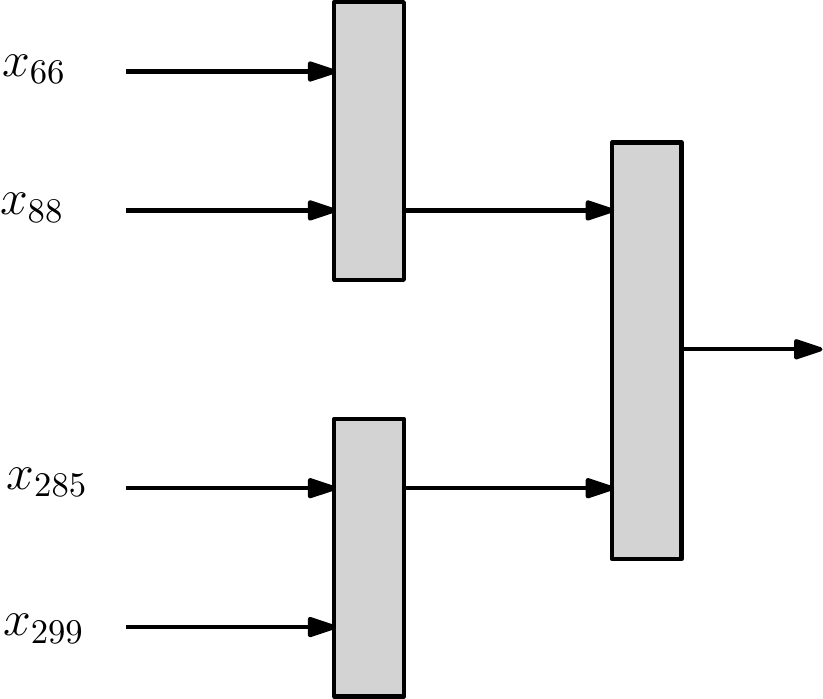}%
		\label{fig_plga_2}%
	}
	\caption{Example--6: Designed HFIT, where the shaded nodes are T2FIS.}
	\label{fig_plga}
\end{figure}

\section{Discussion}
\label{sec_disc}
The proposed HFIT algorithms T1HFIT$^{\text{S}}$, T1HFIT$^{\text{M}}$, T2HFIT$^{\text{S}}$, and T2HFIT$^{\text{M}}$ were evaluated {\color{black}through} six examples{\color{black},} including a real-world {\color{black}example} from the pharmaceutical industry. Performance of the proposed algorithms {\color{black}was} compared with {\color{black}algorithms} that offer {\color{black}fuzzy system's structural} optimization (e.g., SEIT2FNN, SONFIN, SaFIN, TSCIT2FNN, etc.), {\color{black}hierarchical} fuzzy system design (e.g., IFRS, H-TS-FS, etc.), {\color{black}dynamic} fuzzy system design (e.g., DENFIS, D-FNN, etc.), and so forth. The obtained results illustrate the efficiency of the proposed algorithms in comparison to the algorithms collected from the literature. Such performance was obtained by using the parameter setting mentioned in Table~\ref{tab_parameter_fis}. Moreover, a comparison using {\color{black}noisy} data [example 2, case 2 (Section~\ref{sec_perf_ex2})] has proved the approximation efficiency of the proposed algorithms over other algorithms. The HFIT algorithms not only offer {\color{black}solutions} with high accuracy (low approximation error), but they also provide the solutions with low complexity. 

The number of clusters needs to be predetermined in {\color{black}the algorithms that use} a {\color{black}cluster-based} partitioning of the input space and to define fuzzy sets. On the contrary, the proposed HFIT {\color{black}uses a only two partitions for each inputs and automatically defines fuzzy sets} by using the dynamics of the evolutionary process. Such ability is particularly significant for the predictive modeling of {\color{black}problems} like example 6 (Section~\ref{sec_perf_ex6}) that has a large number of input features. It would be a difficult task for {\color{black}fuzzy-NN-based} algorithms (e.g., SONFIN, SEIT2FNN, McIT2FIS, etc.) to design a network-like structure to solve {\color{black}a} high-dimensional problem (e.g., example--6 that has 300 input features){\color{black}, whereas} the proposed HFIT solves example--6 with {\color{black}satisfactory} accuracy and low {\color{black}model's} complexity. Section~\ref{sec_the_comparison} show that HFIT has several qualities that set it apart from many algorithms mentioned in this work. 

In Section~\ref{sec_evaluation_emp}, a comprehensive study of the comparative results of the proposed algorithms was presented. It was observed that the proposed {\color{black}HFIT-based} algorithms gave better performance than the other algorithms collected from the literature. For example, in the case of example--1 T1HFIT$^{\text{M}}$ provided better RMSE with {\color{black}a} lower parameter count. Additionally, T2HFIT$^{\text{M}}$ offered an RMSE (i.e., 0.0028) with a low complexity (i.e., 72) in comparison to {\color{black}SEIT2FNN} that gave an RMSE {\color{black}of} 0.0022 with a model complexity {\color{black}of} 84. 

Similarly, for example-2, T2HFIT$^{\text{M}}$ offered a {\color{black}competitive} RMSE (i.e., 0.0058) in comparison to {\color{black}SEIT2FNN$ ^2 $} that gave an RMSE {\color{black}of} 0.0053. Additionally, {\color{black}in the noisy dataset comparison}, the proposed {\color{black}T2HFIT$^{\text{M}}$} provided better training RMSEs with lower {\color{black}model's} complexities {\color{black}when} compared to many of the recently proposed T2FIS algorithms{\color{black},} such as SEIT2FNN, IT2FNN-SVR, and T2FLS-G. Moreover, the models developed by the proposed algorithm adapted its structure in each instance of {\color{black}noisy dataset experiments}; whereas, the other models had a fixed structure in each instance of their experiments (Table~\ref{tab_mcg_com_noise}). Therefore, the proposed algorithm was able to accommodate the variance in noise more precisely than the other models. 

{\color{black}With} example--3, example--4, and example--5, the proposed type-1 HFIT surpassed all the other algorithms. Whereas, type-2 HFIT performed {\color{black}competitively} with algorithms such as RIT2NFS-WB, McIT2FIS, and SEIT2FNN. It was observed that {\color{black}the} training RMSE of T2HFIT$^{\text{M}}$ for {\color{black}example--3} was {\color{black}as per with} RIT2NFS-WB, but the complexity of the proposed T2HFIT$^{\text{M}}$ was much less than RIT2NFS-WB. For example--4, T2HFIT$^{\text{M}}$ {\color{black}outperformed} all its counterparts in both {\color{black}accuracy} and complexity. For example--5, the training RMSE of T2HFIT$^{\text{M}}$ was {\color{black}close} to SEIT2FNN, but on model complexity and training time, T2HFIT$^{\text{M}}$ outperformed SEIT2FNN by a {\color{black}comfortable} margin. Therefore, it may be concluded that the proposed HFIT version performed efficiently against other algorithms found in the literature.   

{\color{black}The proposed HFIT is a population-based algorithm. Therefore, it should naturally take more training time than a single solution based algorithm. In addition to that, the training time depends on several factors such as the programming language used, the type of platform, the hardware configuration of the machine, {\color{black}the} method of data feeding during the training, etc. Therefore, training time comparison is limited. {\color{black}However, by comparing the training time of the proposed algorithms with the training times of some of other algorithms (training time of only a few algorithms is reported in the literature), the following was observed: 1) in the case of example--1, the proposed T2HFIT$^{\text{M}}$ was found competitive with other algorithms, 2) in the case of {\color{black}example--5}, T2HFIT$^{\text{M}}$ outperformed SEIT2FNN, 3) in the case of example--6, which has 747 samples and 300 input features, the proposed T2HFIT$^{\text{M}}$ takes only about 7.16 minutes, which is remarkable.}}

For example--6, the proposed T2HFIT$^{\text{M}}$ was {\color{black}more} efficient than {\color{black}T1HFIT$^{\text{M}}$} because T2HFIT$^{\text{M}}$ was capable of accommodating {\color{black}noisy} information more efficiently than T1HFIT$^{\text{M}}$. {\color{black}This is evident from the fact that the average RMSE of T2HFIT$^{\text{M}}$ was 16.64, and the average RMSE of T1HFIT$^{\text{M}}$ was 22.36. Hence, {\color{black} the proposed T2HFIT model, which is relied on interval type-2 MFs,} is worth considering in such high-dimensional and noisy application problems.}
	
A comparison between single objective and multiobjective summarized in Table~\ref{tab_s_v_m} {\color{black}suggests} that the multiobjective approach has performance superiority over the single objective because multiobjective gives a competitively better approximation error with lower model complexity in both type-1 and type2 cases {\color{black}compared to single objective}. Additionally, it can be observed that {\color{black}type-2} {\color{black}HFIT} offers better approximation error against type-1 {\color{black}HFIT}.
 \begin{table}
 	\centering
 	\caption{Performance Summary on Benchmark Examples: Single Objective Versus Multiobjective and Type-1 Versus Type-2}
 	\label{tab_s_v_m}
	\begin{tabular}{crrrrrrrrr}
		\toprule
		& \multicolumn{4}{c}{Single Objective} &  & \multicolumn{4}{c}{Multibjective}  \\
		\cline{2-5}\cline{7-10}
		& \multicolumn{2}{c}{T1HFIT$^{\text{S}}$} & \multicolumn{2}{c}{T2HFIT$^{\text{S}}$} & & \multicolumn{2}{c}{T1HFIT$^{\text{M}}$} & \multicolumn{2}{c}{T2HFIT$^{\text{M}}$} \\
		\hline
		Example & $ E_n $ & $ c(\text{\textbf{w}}) $ & $ E_n $  & $ c(\text{\textbf{w}}) $ &   &  $ E_n $ & $ c(\text{\textbf{w}}) $  & $ E_n $ & $ c(\text{\textbf{w}}) $ \\
		\hline
		1 & 0.018 & 57.2 & 0.012 & 152.0 & & 0.025 & 34.4 & 0.018 & 90.4\\
		2 & 0.034 & 71.7 & 0.041 & 203.4 & & 0.033 & 57.6 & 0.022 & 129.5\\
		3 & 2.711 & 132.0 & 2.469 & 224.0 & & 2.603 & 78.8 & 2.405 & 204.4\\
		4 & 2.326 & 77.6 & 2.259 & 188.4 & & 2.428 & 46.4 & 2.242 & 152.9\\
		5 & 0.303 & 138.8 & 0.291 & 286.0 & & 0.344 & 58.4 & 0.301 & 167.4\\
		{\color{black}6} & 24.32 & 220.0 & 16.499 & 208.0 & &17.448 & 156.0 & 14.352 & 108.0\\
		\hline
		Average & 4.952 & 116.2 & 3.595 & 210.3 & & 3.814 & 71.9 & 3.223 & 142.1\\
		\bottomrule
	\end{tabular}
 \end{table} 

Since HFIT algorithms were developed using {\color{black}the} evolutionary process, the quality of {\color{black}their} performance is subjected to carefully setting {\color{black}of} the parameters {\color{black}mentioned in} Table~\ref{tab_parameter_fis}. Hence, the results of the algorithms mentioned in this work may be further improved {\color{black}upon by choosing} different sets of parameters; however, this is a trial-and-error process. For example, the feature selection, i.e., the number of inputs {\color{black}into} a node (a fuzzy subsystem) is proportional to the setting of the maximum inputs {\color{black}into} an node. Similarly, the hierarchy (number of layers) in an HFIT is proportional to the setting of {\color{black}the} maximum depth of a tree. Therefore, {\color{black}HFIT's complexity} can be controlled using these parameters. Additionally, the parameters of MOGP and DE, such as their population size, crossover probability, mutation probability, etc., influence {\color{black}HFIT's performance}.

\section{Conclusions}
\label{sec_con}
{\color{black}Using a} fuzzy inference system (FIS) for data mining inherently requires a multiobjective solution and the proposed multiobjective design {\color{black}for a} hierarchical fuzzy inference tree (HFIT) stands as a viable option that constructs a tree-like model whose nodes are low-dimensional FIS. The proposed  HFIT was developed for both type-1 and type-2 FIS and each node in HFIT implements a Takagi--Sugeno--Kang model. Both type-1 and type-2 FIS {\color{black}were} studied in the scope of single objective and multiobjective optimization using genetic programming. Hence, four versions of HFIT were studied: T1HFIT$^{\text{S}}$, T1HFIT$^{\text{M}}$, T2HFIT$^{\text{S}}$, and T2HFIT$^{\text{M}}$. The parameters of the membership functions and the consequent parts of the rules were {\color{black}optimized} using {\color{black}a} differential evolution algorithm. {\color{black}HFIT's optimization procedure} was a two-phase evolutionary optimization approach, in which structure optimization and parameter optimization {\color{black}were} applied one-by-one until a formidable solution was obtained. {\color{black}The approximation ability of the proposed HFIT was theoretically examined. As a result of that four distinguished quality of HFIT was discovered: adaptation in structure, diverse rule generation, automatic fuzzy set selection, minimal feature drive structure formation.} A comprehensive performance comparison was performed for evaluating the efficiency of the proposed HFIT. {\color{black}The performance of the proposed HFIT algorithm was found {\color{black}to be} {\color{black}efficient} and competitive compared to the {\color{black}algorithms} collected from the literature. {\color{black}HFIT} provided {\color{black}competitive} approximation compared to other algorithms {\color{black}and simultaneously} it produced less complex models. Additionally, HFIT performs feature selection and automatic structure design, which is a necessary for solving high-dimensional problems. } 


\bibliographystyle{IEEEtran}
\bibliography{ieee_fuzzy_ref}
\end{document}